\documentclass[11pt]{article}

\usepackage[T1]{fontenc}
\usepackage[utf8]{inputenc}
\usepackage[letterpaper,margin=1in]{geometry}
\usepackage[numbers,sort&compress]{natbib}

\usepackage{hyperref}
\usepackage{url}
\usepackage{microtype}
\usepackage{graphicx}
\usepackage{booktabs}
\usepackage{amsmath, amsfonts, amssymb}
\usepackage{amsthm}
\usepackage{mathtools}
\usepackage{algorithm}
\usepackage{algpseudocode}
\usepackage{subcaption}
\usepackage{xcolor}
\usepackage{enumitem}
\usepackage{nicefrac}
\usepackage{multirow}
\usepackage{pifont}
\newcommand{\cmark}{\ding{51}}
\newcommand{\xmark}{\ding{55}}
\usepackage{tikz}
\usepackage{pgfplots}
\pgfplotsset{compat=1.18}
\usetikzlibrary{positioning, arrows.meta, fit, backgrounds, calc, patterns}

\newtheorem{theorem}{Theorem}
\newtheorem{proposition}[theorem]{Proposition}

\newtheorem{corollary}[theorem]{Corollary}

\newtheorem{remark}[theorem]{Remark}

\newenvironment{reptheorem}[1]{\noindent\textbf{Theorem~\ref{#1}} (restated)\textbf{.}\itshape}{}

\DeclareMathOperator*{\argmax}{arg\,max}
\DeclareMathOperator*{\argmin}{arg\,min}
\newcommand{\R}{\mathbb{R}}

\title{Differentiable Fuzzy Inference Layer: A Monotone, Compositional Ordinal Reasoning Head for Large Language Models}

\author{%
  Zhen Zhang \qquad Amr Alanwar \\
  School of Computation, Information and Technology \\
  Technical University of Munich, Germany \\
  \texttt{\{zhenzhang.zhang, alanwar\}@tum.de}
}
\date{}

\begin{document}

\maketitle

\begin{abstract}
A state-of-the-art language model asked to interpret “most of most students passed” typically answers “most,” though composing two instances of “most” yields a proportion closer to “some.” We trace this failure to an architectural choice rather than a data deficit: standard classifier heads treat ordinal categories as independent labels, with no mechanism to respect their natural ordering or compose them algebraically. We introduce the Differentiable Fuzzy Inference Layer (DFIL), a dual-path prediction head pairing a standard classifier with a scalar-bottlenecked branch grounded in a bank of ordered membership functions. DFIL supplies two structural primitives that a label-only head cannot inherit: monotonicity in the underlying quantity, and compositional reasoning via t-norm operations without any compositional training data. The scalar branch additionally provides an interpretable interface for analyzing residual errors. We instantiate DFIL on ordinal natural-language tasks across diverse LLM families.
\end{abstract}

\section{Introduction}
\label{sec:intro}

Ordinal reasoning maps a latent continuous quantity to an ordered discrete category, and arises in quantifier interpretation, sentiment intensity, severity grading, and confidence calibration. Asked to interpret “most of most students passed,” a state-of-the-art language model typically answers “most,” although composing two instances of “most” yields a proportion closer to “some.” The failure reproduces across model families, persists under chain-of-thought~\citep{wei2022chain}, and is not closed by fine-tuning on single-step data; the FRoG benchmark~\citep{li2024frog} documents it at scale, reporting inverse scaling on quantifier reasoning consistent with the broader phenomenon~\citep{mckenzie2023inverse}. We trace this failure to the prediction head: a standard classifier treats ordinal categories as independent labels, exposing no explicit scalar on which algebraic composition can be defined, so adding parameters or data cannot remove the gap.

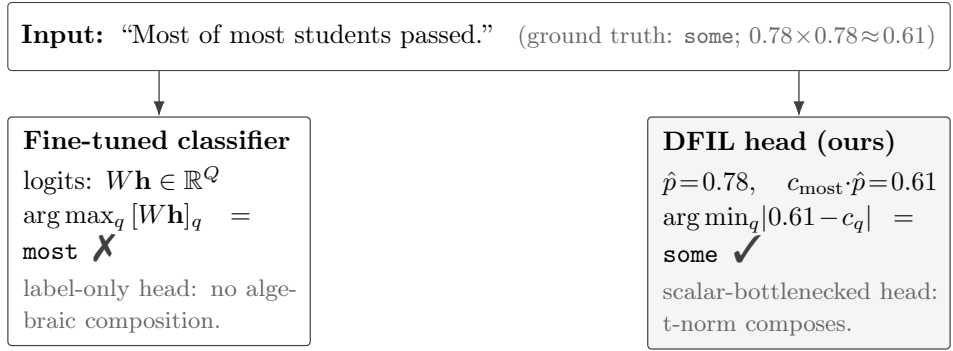
\begin{figure}[t]
\centering
\begin{tikzpicture}[
    font=\small,
    box/.style={draw=black!75, line width=0.45pt, rounded corners=1.5pt,
                inner sep=5pt, align=left, minimum height=0.9cm},
    boxhi/.style={box, fill=black!4},
    arrow/.style={-{Latex[length=2mm]}, line width=0.5pt, draw=black!75},
    note/.style={font=\footnotesize, text=black!65},
]

\node[box, align=center, minimum width=8.6cm, line width=0.4pt] (prompt)
  {\textbf{Input:}\; “Most of most students passed.”
   \;\;{\footnotesize\color{black!60}(ground truth: \texttt{some}; $0.78\!\times\!0.78\!\approx\!0.61$)}};

\node[box, below=0.6cm of prompt.south west, anchor=north west,
      minimum width=4cm, text width=3.6cm] (ft)
  {\textbf{Fine-tuned classifier}\\[3pt]
   logits: $W\mathbf{h}\in\mathbb{R}^Q$\\[1pt]
   $\arg\max_q\,[W\mathbf{h}]_q\;=\;\texttt{most}$\;\;{\Large\color{black!75}\xmark}\\[2pt]
   {\footnotesize\color{black!60}label-only head: no algebraic composition.}};

\node[boxhi, below=0.6cm of prompt.south east, anchor=north east,
      minimum width=4cm, text width=3.6cm] (dfil)
  {\textbf{DFIL head (ours)}\\[3pt]
   $\hat p\!=\!0.78,\quad c_{\text{most}}\!\cdot\!\hat p\!=\!0.61$\\[1pt]
   $\arg\min_q\!|0.61\!-\!c_q|\;=\;\texttt{some}$\;\;{\Large\color{black!75}\cmark}\\[2pt]
   {\footnotesize\color{black!60}scalar-bottlenecked head: t-norm composes.}};

\draw[arrow] (prompt.south -| ft.north) -- (ft.north);
\draw[arrow] (prompt.south -| dfil.north) -- (dfil.north);

\end{tikzpicture}
\caption{\textbf{Composing two ``most'' should yield ``some,'' not ``most.''} A label-only fine-tuned classifier picks $\arg\max_q[W\mathbf{h}]_q\!=\!\texttt{most}$ (incorrect); DFIL composes via t-norm on learned ordered centers $c_q$ and decides by nearest-center rule (Thm.~\ref{thm:compositionality}, \ref{thm:monotonicity}), yielding $\texttt{some}$. \textit{The label-only head has no analogous algebraic primitive.} Shaded box marks our method.}
\label{fig:teaser}
\end{figure}

Ordinal prediction in language obeys two structural constraints standard architectures cannot enforce: \textit{monotonicity} (predicted class is non-decreasing in the latent quantity) and \textit{compositionality} (nested expressions like “most of most” compose from their components). Transformer classifiers operate on high-dimensional hidden representations without an explicit scalar, so they cannot enforce either constraint and empirically violate both. Fuzzy set theory~\citep{zadeh1965fuzzy} supplies the canonical primitives---ordered membership functions and algebraic t-norm composition---but classical fuzzy systems are hand-tuned in isolation and have not been integrated into modern LLM heads.

We propose the Differentiable Fuzzy Inference Layer (DFIL), a dual-path prediction head. The main path is a standard classifier; the DFIL path extracts a scalar $\hat{p}\in[0,1]$ via a small numerical head and projects it onto a bank of ordered Gaussian membership functions, learned end-to-end with the backbone. The MF branch regularizes training and, at inference, yields predictions that are monotone by construction and compose via t-norm operations without any compositional training data. Prior ordinal-regression heads enforce ordering via cumulative logits or unimodal outputs but cannot compose~\citep{cao2020rank,shi2023corn,beckham2017unimodal}; DFIL operates in an interpretable proportion space where algebraic composition is available on the head's output.

\textbf{Contributions.} (1) We introduce DFIL, an end-to-end-trained LLM head that combines learnable ordered Gaussian membership functions with a t-norm composition rule, providing monotonicity by construction and algebraic composition at the head level. (2) We prove two structural guarantees: monotone nearest-center inference (Thm.~\ref{thm:monotonicity}) and compositional closure under proportional shift (Thm.~\ref{thm:compositionality}). (3) Empirically, DFIL matches fine-tuning on single-step accuracy across six LLM backbones while measurably improving compositional reasoning, low-budget sample efficiency, and natural-language monotonic consistency. (4) An initialization-robustness experiment (\S\ref{sec:uniform_init}) isolates the fuzzy training signal as the causal mechanism behind the compositional advantage.

\section{Related Work}

The structural failures of LLMs on ordinal reasoning call for head-level primitives that a label-only classifier cannot inherit. Fuzzy set theory supplies these primitives---ordered MFs, t-norm composition, monotone decision rules---studied since \citet{zadeh1965fuzzy} and combined with gradient-based learning by ANFIS~\citep{jang1993anfis}, but not brought into LLM heads as a trained differentiable layer.

\textbf{Quantifier reasoning in NLP.} Fuzzy quantifiers have a long history in linguistics and formal semantics~\citep{zadeh1983role,barwise1981generalized,vanbenthem1986essays,pezzelle2018some}. FRoG~\citep{li2024frog} reveals systematic LLM failures with inverse scaling. PRESQUE~\citep{li2023presque} applies pragmatic Rational Speech Acts; our re-implementation scores $>25$pp below DFIL on FRoG-Hard. LaSQuE~\citep{ghosh2023lasque} uses NLI-explanation supervision and ordinal constraints; its ordinal component is most closely paralleled by our CORAL~\citep{cao2020rank} baseline. Both operate at the sentence level without proportion-level grounding or composition; DFIL grounds quantifier semantics directly in proportion space.

\textbf{Inductive biases for reasoning.} Monotonic architectures for ordinal regression include cumulative-logit (CORAL~\citep{cao2020rank}, CORN~\citep{shi2023corn}), unimodal outputs~\citep{beckham2017unimodal}, and constrained monotonic networks~\citep{sill1998monotonic,you2017deep,wehenkel2019unconstrained,runje2023constrained,sartor2025monotonic,gutierrez2016ordinal}; none compose. Compositional reasoning failures in LLMs are documented by GSM-Symbolic~\citep{mirzadeh2025gsm} and chained GSM~\citep{hosseini2025compositional}. Logic Tensor Networks~\citep{badreddine2022ltn} and NeSyCoCo~\citep{kamali2025nesycoco} integrate fuzzy/symbolic reasoning at the relational level rather than as an LLM head; \citet{chen2022fuzzy} apply fuzzy operators to knowledge graphs. Other structural priors (positional encodings~\citep{press2022train}, graph networks~\citep{battaglia2018relational}) target distinct constraints.

\textbf{Positioning.} DFIL builds on LoRA~\citep{hu2022lora,dettmers2023qlora}, adding ${\sim}0.9$--$2.1$M parameters depending on backbone size. The differentiator is the joint set of head-level capabilities (Tab.~\ref{tab:relwork_compare}): an end-to-end-trained LLM head with learnable ordered centers, monotone by construction, algebraically composing, and interpretable. DFIL is the only method satisfying all six capabilities simultaneously.

\begin{table}[!htb]
\centering
\caption{Closest baselines on head-level capabilities for ordinal and compositional reasoning. ``Explicit $p$'' = head exposes a scalar proportion; ``Ordered $c_q$'' = learnable, order-constrained class centers; ``Mono.\ by const.'' = predictions monotone in $p$ by construction (deployed inference rule); ``Composes'' = supports algebraic composition without re-training; ``LLM head'' = used as a Transformer LLM head in published work; ``Interpretable'' = per-class quantities mapping to human-readable boundaries.}
\label{tab:relwork_compare}
\footnotesize
\setlength{\tabcolsep}{3pt}
\begin{tabular}{lcccccc}
\toprule
\textbf{Method} & \textbf{$p$ exposed} & \textbf{Ordered $c_q$} & \textbf{Mono.\ by const.} & \textbf{Composes} & \textbf{LLM head} & \textbf{Interpretable} \\
\midrule
LoRA fine-tune (label head)        & \xmark & \xmark & \xmark & \xmark & \cmark & \xmark \\
CORAL~\citep{cao2020rank}          & \xmark & \xmark & \cmark$^\dagger$ & \xmark & partial$^\ddagger$ & \xmark \\
LaSQuE~\citep{ghosh2023lasque}     & \xmark & \xmark & \xmark & \xmark & \cmark & partial \\
PRESQUE~\citep{li2023presque}      & \cmark$^\S$ & fixed & \cmark$^\S$ & \xmark & no$^\P$ & \cmark \\
ANFIS~\citep{jang1993anfis}        & \cmark & \cmark & \cmark & \cmark$^*$ & \xmark & \cmark \\
\midrule
\textbf{DFIL (ours)}               & \cmark & \cmark & \cmark & \cmark & \cmark & \cmark \\
\bottomrule
\end{tabular}
\\[2pt]
{\footnotesize $^\dagger$Cumulative-logit ordering, no algebraic composition. $^\ddagger$Generic ordinal head, not standard as Transformer LLM head. $^\S$Inference-time RSA over fixed QuRe centers, not jointly learned. $^\P$Post-hoc over LLM probabilities, not end-to-end head. $^*$T-norm composition, but ANFIS targets low-D control inputs not LLM hidden states.}
\end{table}

\section{Problem Formulation}
\label{sec:problem}

\subsection{Fuzzy Quantifier Reasoning}

Given a statement with a masked quantifier slot, as in “\_\_\_ of 20 students passed; specifically 15 of 20,” the task predicts a label from $\mathcal{Q}=\{q_0,\ldots,q_{Q-1}\}$ ordered by magnitude ($q_0{=}$“none” $< q_1{=}$“tiny” $<\cdots< q_7{=}$“all”). Each example has a proportion $p\in[0,1]$. A correct prediction function $Q:[0,1]\to\{0,\ldots,Q{-}1\}$ must satisfy monotonicity: $p_1<p_2 \Rightarrow Q(p_1)\leq Q(p_2)$.\label{eq:monotonicity}

\subsection{The Expressiveness Gap}

We formalize two structural limitations of standard classifiers; full proofs are in Appendix~\ref{app:proofs}.

\begin{theorem}[Monotonicity Guarantee, nearest-center inference rule]
\label{thm:monotonicity}
Let $c_0 < c_1 < \cdots < c_{Q-1}$ be ordered centers in $[0,1]$. Then the nearest-center decision rule
\[
  \hat{Q}_{\mathrm{MF}}(p) \;=\; \argmin_{q\in\{0,\dots,Q-1\}}\;|p-c_q|
\]
is monotonically non-decreasing in $p$. We adopt this rule for the MF branch's inference-time predictions; equivalently, for any input grid the prediction has zero pairwise violations by construction.
\end{theorem}

The proof is a midpoint-cut: ordered centers induce decision boundaries $b_q{=}(c_q{+}c_{q+1})/2$. The Gaussian forward $\mu_q(p)$ supplies a differentiable training signal; the nearest-center rule applies at inference. Conversely, no linear head over a $d{\geq}2$ hidden state can guarantee monotonicity (Proposition~\ref{prop:classifier_fails}, App.~\ref{app:proofs}).

\subsection{The Compositionality Gap}

Compositional quantifier reasoning (``most of most'') combines two quantifier applications: ``$Q_1$ of $Q_2$'' at proportion $p$ requires $\mu_{Q_1\circ Q_2}(p)=T(\mu_{Q_1}(p),\mu_{Q_2}(p))$ where $T$ is a t-norm (e.g., product $T(a,b)=a\cdot b$). We deploy a proportion-level form $p\mapsto c_{Q_1}p$ (Sec.~\ref{sec:compositional_method}), the product t-norm at the proportion level, and state closure for this form.

\begin{theorem}[Compositional Closure for Proportional-Shift Inference]
\label{thm:compositionality}
Let $c_0 < c_1 < \cdots < c_{Q-1} \in (0,1)$ be ordered MF centers, and define the $n$-step proportional-shift composition operator
\[
  Q^{(n)}\!\left(q_1,\ldots,q_{n-1};\,p\right) \;:=\; \argmin_{q\in\{0,\ldots,Q-1\}}\,\bigl|\,c_{q_1} c_{q_2}\cdots c_{q_{n-1}}\,p \;-\; c_q\bigr|,
\]
i.e., the proportional-shift inference rule (Sec.~\ref{sec:compositional_method}) iterated $n{-}1$ times. Then: (i) $Q^{(n)}$ is well-defined for every $n\geq 1$ and $p\in[0,1]$, with ties broken by smallest index; (ii) the composed proportion $c_{q_1}\cdots c_{q_{n-1}}\,p\in[0,p]$ is non-increasing in $n$, so deeper composition never predicts a class strictly above the base, i.e., “most of most” $\leq$ “most” as a quantifier index; (iii) the prediction depends only on $\{c_q\}$ (not on widths $\{\sigma_q\}$) and is monotone in $p$ at every depth via Theorem~\ref{thm:monotonicity}.
\end{theorem}

A label-only classifier $f_\theta:\mathcal{X}\to\R^Q$ exposes neither a proportion scalar nor learned class anchors, so it admits no by-construction $Q^{(n)}$ operator (Remark~\ref{rem:classifier_nocomp}); classifiers can still approximate compositional behaviour given supervised compositional data. Proof in App.~\ref{app:proofs}; alternative t-norms in App.~\ref{app:tnorm}.

\subsection{Universality and Complexity}

Ordered Gaussian MFs exactly realize any $Q$-step monotone target on $[0,1]$ (Prop.~\ref{prop:universality}, App.~\ref{app:proofs}). With $Q{=}8$ the MF has $2Q{-}1{=}15$ effective parameters vs.\ $\Theta(Qd)$ for a linear head over $d{\geq}3584$ ($\sim\!1900\times$ reduction); head-level VC ratio of $O(Q\log Q)$ vs.\ $\Theta(Qd)$~\citep{vapnik1998statistical} exceeds $400\times$ (Prop.~\ref{prop:sample_complexity}). With a LoRA-adapted backbone dominating joint complexity (App.~\ref{sec:fewshot}), DFIL's key benefit is structural guarantees rather than raw sample efficiency.

\section{Method: Differentiable Fuzzy Inference Layer}
\label{sec:method}

\subsection{Architecture Overview}

DFIL is a dual-path prediction head on top of a LoRA-adapted frozen LLM backbone (Figure~\ref{fig:architecture}). The shared hidden state $\mathbf{h}=\text{LLM}(x)$ feeds two heads in parallel: a standard linear classifier producing logits $\hat y\in\R^Q$, and a numerical head extracting a scalar proportion $\hat p\in[0,1]$ that is mapped through a bank of $Q$ ordered Gaussian membership functions to class-wise memberships $\boldsymbol\mu\in\R^Q$.
\begin{equation}
\underbrace{\mathbf{h} = \text{LLM}(x)}_{\text{backbone}} \;\to\;
\begin{cases}
  \text{Main:} & \hat{y} = \text{Classifier}(\mathbf{h}) \\
  \text{DFIL:} & \hat p = \text{NumHead}(\mathbf{h}), \;\; \boldsymbol{\mu} = \text{MF}(\hat p).
\end{cases}
\end{equation}

\begin{figure}[t]
\centering
\begin{tikzpicture}[
    font=\small,
    box/.style={draw=black!75, line width=0.5pt, rounded corners=1pt,
                inner sep=4pt, minimum height=0.7cm, align=center, fill=white},
    boxhi/.style={box, fill=black!4},
    frozen/.style={box, pattern=north east lines, pattern color=black!22},
    var/.style={font=\small},
    arrow/.style={-{Latex[length=2mm]}, line width=0.5pt, draw=black!75},
    note/.style={font=\footnotesize, text=black!65},
]

\node[box, minimum width=1.3cm] (input) {Input $x$};
\node[frozen, right=0.45cm of input, minimum width=2.2cm] (llm) {LLM\,+\,LoRA};
\node[var, right=0.3cm of llm] (h) {$\mathbf{h}$};

\draw[arrow] (input) -- (llm);
\draw[arrow] (llm) -- (h);

\node[box, right=0.5cm of h, yshift=0.6cm, minimum width=1.6cm] (cls) {Classifier};
\node[var, right=0.3cm of cls] (yhat) {$\hat y\!\in\!\mathbb{R}^Q$};

\node[boxhi, right=0.5cm of h, yshift=-0.6cm, minimum width=1.6cm] (num) {NumHead};
\node[var, right=0.25cm of num] (p) {$\hat p$};
\node[boxhi, right=0.25cm of p, minimum width=1.9cm] (mf) {Ordered MF};
\node[var, right=0.25cm of mf] (mu) {$\boldsymbol{\mu}\!\in\!\mathbb{R}^Q$};
\node[boxhi, right=0.25cm of mu, minimum width=1.1cm] (tnorm) {$\otimes$ t-norm};

\draw[arrow] (h.east) -- ++(0.15,0) |- (cls.west);
\draw[arrow] (h.east) -- ++(0.15,0) |- (num.west);
\draw[arrow] (cls) -- (yhat);
\draw[arrow] (num) -- (p);
\draw[arrow] (p) -- (mf);
\draw[arrow] (mf) -- (mu);
\draw[arrow, dashed, draw=black!55] (mu) -- (tnorm);

\node[note, above=0.15cm of cls] {main path};
\node[note, below=0.15cm of num] {DFIL path\;\,(15 MF params)};
\node[note, below=0.05cm of tnorm.south, anchor=north] {(inference)};

\end{tikzpicture}
\caption{\textbf{DFIL dual-path architecture.} LoRA-adapted backbone (hatched) produces hidden state $\mathbf{h}$. Main path: classifier $\mathbf{h}\to\hat y$. DFIL path (shaded): NumericalHead $\mathbf{h}\to\hat p\!\in\![0,1]$, projected through $Q$ ordered Gaussian MFs to $\boldsymbol{\mu}$, with t-norm (dashed) for compositional and entailment queries. Both paths trained jointly; inference path selected by downstream needs (Thm.~\ref{thm:monotonicity}, \ref{thm:compositionality}).}
\label{fig:architecture}
\end{figure}
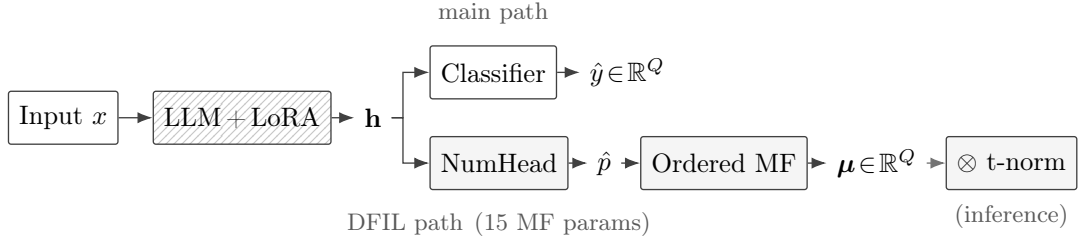

The main classifier integrates the full hidden state and suits raw single-step accuracy. The DFIL path bottlenecks through a 1D scalar; the centers $c_0{<}\cdots{<}c_{Q-1}$ partition that scalar into nearest-center decision cells, so increasing $\hat p$ traverses class indices monotonically (Thm.~\ref{thm:monotonicity}), and algebraic operations on the scalar (multiplication for composition, thresholds for entailment) execute directly on the head's output. We keep both heads and select between them at inference time (Sec.~\ref{sec:monotonicity_analysis}).

\subsection{Ordered Gaussian Membership Function}
\label{sec:method:mf}

The core of DFIL is an ordered Gaussian MF with $Q = 8$ quantifier classes. Each quantifier $q$ has a center $c_q$ and width $\sigma_q$:
\begin{equation}
  \mu_q(p) = \exp\!\left(-\frac{(p - c_q)^2}{2\sigma_q^2}\right), \quad q = 0, \ldots, Q-1
\end{equation}

\textbf{Ordered parameterization.} Centers are produced by cumulative softmax over learnable deltas $\boldsymbol{\delta}_{\text{raw}}\in\R^Q$:
\begin{equation}
  \boldsymbol{\delta}=\mathrm{softmax}(\boldsymbol{\delta}_{\text{raw}})+\epsilon,\;\; \boldsymbol{\delta}\!\leftarrow\!\boldsymbol{\delta}/\|\boldsymbol{\delta}\|_1,\;\; \mathbf{c}=\mathrm{cumsum}(\boldsymbol{\delta})\cdot 0.96+0.02
  \label{eq:ordered_centers}
\end{equation}
with $\epsilon{=}10^{-4}$. This guarantees $c_0{<}\cdots{<}c_{Q-1}\in[0.02,0.98]$. Since $\sum_q\delta_q{=}1$ after renormalization, $c_{Q-1}$ is fixed at $0.98$; only $c_0,\ldots,c_{Q-2}$ are learnable. Widths are log-space: $\sigma_q=\exp(w_q)$. The MF has $2Q{-}1{=}15$ effective parameters and is monotonic by construction. Centers are initialized from QuRe~\citep{li2023presque}; the ordering constraint ensures validity under arbitrary initialization (Sec.~\ref{sec:uniform_init}).

\textbf{Inference rule.} At inference the MF branch uses the nearest-center rule $\hat q(p)=\argmin_q|p-c_q|$, which depends only on centers (not widths). This makes Thm.~\ref{thm:monotonicity} tight: $\argmax_q\mu_q$ can introduce isolated tail crossings under non-uniform widths (one violation per $10^6$ pairs in our 7B model), whereas nearest-center has $0\%$ violations on any grid.

\subsection{Numerical Head and Training Objective}

The numerical head is a two-layer MLP, $\mathrm{LN}(d)\to\mathrm{Linear}(d{\to}256)\to\mathrm{GELU}\to\mathrm{Linear}(256{\to}1)\to\sigma$, extracting $\hat p\in[0,1]$ from the last-token state. Training:
\begin{equation}
  \mathcal{L}=\mathcal{L}_{\mathrm{CE}}(\hat y,y)+\lambda_{\mathrm{DFIL}}\,\mathcal{L}_{\mathrm{CE}}(\boldsymbol{\mu},y)+\lambda_p\|\hat p-p\|_2^2,\quad \boldsymbol{\mu}=\mathrm{MF}(g(\hat p)),
  \label{eq:loss}
\end{equation}
with $g(\cdot)$ the identity for full-data and a stop-gradient for $N{\leq}256$ few-shot (App.~\ref{sec:fewshot}; prevents sigmoid saturation under class imbalance). Defaults: $\lambda_{\mathrm{DFIL}}{=}\lambda_p{=}0.5$ (ablation regime); the parity table uses $\lambda_{\mathrm{DFIL}}{=}0.1$, validation-tuned. Three parameter groups: backbone LoRA $2{\times}10^{-5}$, heads $10^{-3}$, MF $10^{-2}$. Widths initialized at $w_q=\log 0.1$.

\subsection{Compositional Inference}
\label{sec:compositional_method}

For ``$Q_1$ of $Q_2$'' at base proportion $p$, we apply a proportional shift by the learned center $c_{Q_1}$ ($c_{\text{most}}{\approx}0.78$ after training) and classify through the MF: $p_{\text{comp}}=c_{Q_1}\cdot p$, $\hat y_{Q_1\circ Q_2}=\argmax_q\mu_q(p_{\text{comp}})$. For $n$-step chains, $p_{\text{comp}}=c_{Q_1}c_{Q_2}\cdots p$. Thm.~\ref{thm:compositionality} establishes closure for this proportional-shift operator, which preserves proportion-space interpretability and inherits monotonicity from Thm.~\ref{thm:monotonicity}. Equivalence to the membership-level fuzzy-AND at decision boundaries and alternative t-norms are in App.~\ref{app:tnorm}. The capability requires no compositional training data; it follows from the centers' semantic grounding.

\section{Experiments}
\label{sec:experiments}

\subsection{Setup}
\label{sec:setup}

We evaluate on FRoG~\citep{li2024frog} ($2{,}044$ examples/split, $8$ quantifier classes), training on the “easy” split (80/20 train/val) and testing on “hard.” Backbones: Qwen2.5-Instruct \{7B, 14B, 32B, 72B\}~\citep{qwen2024qwen25}, Mistral-7B-v0.3~\citep{jiang2023mistral}, Llama-3.1-8B~\citep{grattafiori2024llama}, Phi-3.5-mini~\citep{abdin2024phi3}. Frozen with LoRA ($r{=}16$, $\alpha{=}32$, $q,v$ projections; $q,k,v,o$ for cross-task in Sec.~\ref{sec:crosstask}); 72B uses QLoRA 4-bit~\citep{dettmers2023qlora}. Training: 20 epochs, batch $8$, weighted CE, AdamW, warmup--cosine, three seeds on H200. Full hyperparameters in App.~\ref{app:details}.

\textbf{Headline.} Table~\ref{tab:headline} summarizes the six empirical findings that anchor this paper, spanning compositional reasoning, single-step parity, sample efficiency, natural-language consistency, and by-construction monotonicity. The rest of \S\ref{sec:experiments} provides per-experiment context.

\begin{table}[!htb]
\centering
\footnotesize
\caption{\textbf{Headline empirical results.} Six findings anchoring the paper's contributions; full per-cell numbers in the indicated table or appendix. Multi-seed cells report sample std (ddof=1).}
\label{tab:headline}
\setlength{\tabcolsep}{4pt}
\begin{tabular}{lll}
\toprule
\textbf{Claim} & \textbf{Effect} & \textbf{Statistic / setting} \\
\midrule
Compositional from-text (Tab.~\ref{tab:compositional})         & $\mathbf{+5.9}$pp     & paired $t$ $p{=}0.011$, $d{=}0.63$, $n{=}20$ seeds \\
Uniform-init compositional (\S\ref{sec:uniform_init})           & $\mathbf{+23.2}$pp    & alignment-controlled, $n{=}3$ seeds \\
Single-step parity, 6 backbones (App.~\ref{app:scaling_full})   & $-0.1$ to $+1.9$pp    & Welch $p\!\in\![0.43, 0.97]$, FRoG-Hard \\
Sample efficiency, $N\!\le\!64$ (App.~\ref{app:fewshot})        & $+11.7$ to $+16.5$pp  & MF branch dominates at low budget \\
CSMC consistency, $n{=}21$ (App.~\ref{app:csmc_domains})        & $\mathbf{+5.50}$pp    & paired $t$ $p{=}0.0001$, Wilcoxon $p{=}0.0003$ \\
\bottomrule
\end{tabular}
\end{table}

\textbf{Benchmark vs.\ diagnostic probes.} External benchmarks (FRoG, QuRe~\citep{li2023presque}, SST-5~\citep{socher2013sst}, STS22) are pre-existing held-out evaluation; numbers there are direct capability claims. Diagnostic probes (IPR, Cross-Sentence Monotonic Consistency, naturalistic templates, entailment pairs) are author-constructed to isolate architectural properties; some are by-construction (e.g., MF entailment $=100\%$) and serve as sanity checks.

\subsection{Compositional Quantifier Reasoning}
\label{sec:compositional}

We test compositional reasoning on $900$ unseen examples ($600$ two-step like “most of most,” $300$ three-step like “few of most of most”), with GT labels via $p_{\text{comp}}=c_{Q_1}\cdots p_{\text{base}}$ through QuRe centers. All models train only on single-step FRoG. Three inference modes: \textit{Fine-tune text-only}, \textit{DFIL-from-text} ($\hat p$-extract then t-norm composes; fair comparison), and \textit{DFIL-oracle} (composition on ground-truth $p$; isolates compositional mechanism from extraction noise).

\begin{table}[!htb]
\centering
\footnotesize
\caption{\textbf{Compositional quantifier reasoning}, Qwen2.5-7B, $900$ unseen compositions. We report two evaluation protocols. \textit{QuRe-based GT} (literature default) labels from frozen QuRe centers and favours methods whose centers sit at QuRe. \textit{Self-consistent GT} labels from each model's own learned centers and favours methods with adaptive centers; it is leakage-free w.r.t.\ external generators but not w.r.t.\ the predictor itself. Oracle$^{\dagger}$: composition applied to ground-truth $p$. Under QuRe-based GT, FT+Compose's oracle ($94.3$) reflects its frozen QuRe-aligned centers mechanically matching the GT generator (not a capability advantage); DFIL's lower oracle there ($50.5$) reflects its learned-center drift from QuRe (not compositional weakness). The self-consistent oracle row above confirms both methods reach $100\%$ against their own MF (Theorem~\ref{thm:compositionality}).}
\label{tab:compositional}
\begin{tabular}{lcc}
\toprule
\textbf{Method} & \textbf{From-text} & \textbf{Oracle$^{\dagger}$} \\
\midrule
\multicolumn{3}{l}{\textit{Self-consistent GT} (leakage-free, $n{=}20$ seeds for FT+Compose / DFIL)} \\
Fine-tune (text-only)        & $16.4 \pm 7.3^{\natural}$ & --- \\
FT+Compose                   & $25.4 \pm 3.5$ & $100.0 \pm 0.0^{\flat}$ \\
\textbf{DFIL (ours)}         & $\mathbf{31.2 \pm 9.2}$ & $100.0 \pm 0.0^{\flat}$ \\
$\Delta$ vs.\ Fine-tune      & $+14.8$ & --- \\
$\Delta$ vs.\ FT+Compose     & $\mathbf{+5.9}^{\ddagger}$ & $0.0$ \\
\midrule
\multicolumn{3}{l}{\textit{QuRe-based GT} (literature default, has leakage)} \\
Fine-tune (text-only)        & $16.4 \pm 7.3$ & --- \\
FT+Compose                   & $27.8 \pm 3.6$ & $94.3 \pm 0.0$ \\
DFIL                         & $31.5 \pm 0.3$ & $50.5 \pm 13.0$ \\
$\Delta$ DFIL $-$ FT+Compose & $+3.7$ & --- \\
\bottomrule
\multicolumn{3}{l}{\footnotesize $^{\ddagger}$ $+5.9$pp significant at $n{=}20$ (paired $t$ $p{=}0.011$, Cohen's $d{=}0.63$).} \\
\multicolumn{3}{l}{\footnotesize $^{\flat}$ $100\%$ on the self-consistent oracle is by construction (each model vs.\ its own MF).} \\
\multicolumn{3}{p{0.95\columnwidth}}{\footnotesize $^{\natural}$ Fine-tune has no MF, so its predictions are protocol-independent; the same $n{=}3$ archive value is reported in both blocks.} \\
\end{tabular}
\end{table}

Both protocols are asymmetric: QuRe-based GT favours frozen QuRe-aligned centers (FT+Compose), self-consistent GT favours adaptive centers (DFIL). DFIL leads in both ($+3.7$pp QuRe-based, $+5.9$pp self-consistent at $n{=}20$, paired $t$ $p{=}0.011$); the uniform-init ablation (\S\ref{sec:uniform_init}) widens the gap to $+23.2$pp when neither alignment advantage is in play, isolating the fuzzy training signal as the causal mechanism. Self-consistent oracle confirms Thm.~\ref{thm:compositionality} ($100\%$ for both methods). DFIL's wider self-consistent std ($\pm 9.2$ vs.\ $\pm 3.5$) is partly mechanical: each seed defines both labels and predictions (App.~\ref{app:variance_analysis}).

\subsection{Initialization Robustness, Entailment, and Naturalistic Transfer}
\label{sec:uniform_init}

\textbf{Uniform vs.\ QuRe initialization.} The QuRe-based GT risks circularity for FT+Compose. Repeating \S\ref{sec:compositional} with uniform MF init (evenly in $[0.02,0.98]$) magnifies DFIL's advantage (Tab.~\ref{tab:fair_comp}): DFIL reaches $41.6{\pm}7.3\%$ from-text (vs.\ $31.2{\pm}9.2$ at QuRe init), FT+Compose collapses to $18.4{\pm}6.2\%$, widening the DFIL$-$FT+Compose gap to $+23.2$pp, confirming that the fuzzy training signal (not center alignment) drives compositional learning. Single-step accuracy is initialization-insensitive. Per-seed results in App.~\ref{app:uniform_init_table}.

\begin{table}[!htb]
\centering
\footnotesize
\caption{Self-consistent compositional accuracy (\%); GT from each model's own MF (oracle $=100\%$ by construction). Mean $\pm$ std, Qwen2.5-7B: uniform $n{=}3$, QuRe $n{=}20$. QuRe from-text paired $t$ $p{=}0.011$, $d{=}0.63$.}
\label{tab:fair_comp}
\begin{tabular}{llcc}
\toprule
\textbf{Init} & \textbf{Method} & \textbf{From-text} & \textbf{Text-only} \\
\midrule
\multirow{2}{*}{Uniform} & DFIL       & $\mathbf{41.6 \pm 7.3}$ & $13.4 \pm 4.5$ \\
                          & FT+Compose & $18.4 \pm 6.2$          & $19.5 \pm 4.1$ \\
\midrule
\multirow{2}{*}{QuRe}    & DFIL       & $31.2 \pm 9.2$          & $18.5 \pm 9.2$ \\
                          & FT+Compose & $25.4 \pm 3.5$          & $17.6 \pm 3.6$ \\
\bottomrule
\end{tabular}
\end{table}

DFIL's lower uniform text-only cell ($13.4$ vs $19.5$) reflects the absence of the explicit fraction signal: without it the NumericalHead has nothing to extract and the MF branch defaults to its center-of-mass, while a label-only classifier can still memorize surface-template priors. The from-text column, where DFIL receives the proportion as input, is the regime the dual-path design targets, and that is where the $+23.2$pp uniform-init gap appears.

\label{sec:naturalistic_external}
\textbf{Quantifier entailment.} On $1{,}344$ ordered pairs spanning clinical, election, and manufacturing scenarios, DFIL's MF-based entailment test ($c_{Q_i}{\leq}c_{Q_j}$) is $\mathbf{100\%}$ correct by construction (direct empirical instantiation of Theorem~\ref{thm:monotonicity}); classifier-based entailment is near chance (Tab.~\ref{tab:entailment}).

\textbf{Naturalistic compositional sentences.} On $320$ author-curated templates evaluated against bias-free QuRe-based labels, DFIL-from-text reaches $39.5{\pm}3.3\%$ vs.\ $11.1{\pm}0.9\%$ for FT text-only, a $+28.4$pp gap (Tab.~\ref{tab:naturalistic}).


To test transfer beyond templates, we mined $249$ naturally-occurring sentences from GenericsKB-Best~\citep{bhakthavatsalam2020genericskb}, PubMedQA~\citep{jin2019pubmedqa}, and CNN/DailyMail~\citep{see2017cnndm} via regex on “$Q_1$ (of) class (who$|$that$|\cdots$) clause,” filtered by a 4-judge ensemble spanning model families and scales (Qwen2.5-7B, Mistral-7B-Instruct-v0.3, Llama-3.1-8B-Instruct, and Qwen2.5-72B-Instruct) with per-judge triple-vote, retaining items with $\ge 3$ of $4$ judges in agreement (pairwise judge agreement $70$--$81\%$ across $6$ pairs; full pipeline and per-judge bias analysis in App.~\ref{app:external_nl_breakdown}). FT+Compose reaches $36.4{\pm}1.9\%$ vs.\ $12.9{\pm}0.4\%$ text-only ($\mathbf{+23.6}$pp); DFIL reaches $32.8{\pm}3.7\%$ vs.\ $12.6{\pm}0.6\%$ ($+20.2$pp). FT+Compose's $3.6$pp lead over DFIL on this benchmark reflects an alignment confounder: its frozen QuRe-aligned centers match the GT generator. The same confounder is controlled in the uniform-init experiment above, where DFIL leads by $+23.2$pp. Both methods' text-only modes near random rule out representation-quality confounds. The natural-prose from-text gap is $4.8$pp smaller than on author templates, comparable in magnitude across all three corpora ($+20.6$ to $+26.1$pp) and four populated $Q_1$ classes ($+11.6$ to $+39.1$pp).

\subsection{Cross-Benchmark Transfer: QuRe}
\label{sec:qure}

Evaluating FRoG-trained checkpoints on QuRe~\citep{li2023presque} (401 Wikipedia sentences, four-quantifier set, no retraining): on Qwen-7B the MF branch reaches $31.3{\pm}2.7\%$, $+14.7$pp over fine-tuning ($16.6{\pm}4.4$; Welch's $t$ $p{\approx}0.012$, single-cell); on Mistral-7B the main classifier leads by $+2.5$pp ($19.0{\pm}2.3$ vs.\ $16.5{\pm}0.9$, $p{\approx}0.20$). Fine-tuning has no analogous fallback path; full per-cell numbers in App.~\ref{app:qure_full}.

\subsection{Single-Step Parity and Ablation}
\label{sec:scaling}

\textbf{Single-step parity across scales and families.} On full-data single-step FRoG-Hard across Qwen2.5-\{7B, 14B, 32B, 72B\}, Mistral-7B, and Llama-3.1-8B (per-backbone $n{=}3$--$6$; per-seed results in App.~\ref{app:scaling_full}), per-backbone mean gaps between DFIL and fine-tuning span $-0.1$ to $+1.9$pp (Tab.~\ref{tab:scaling}). DFIL incurs at most $0.1$pp loss in two of six backbones (Qwen-7B, Llama-8B), gains at most $1.9$pp in one (Mistral-7B), and is within $\pm 1$pp on the remaining four. Welch tests do not reject equality at any conventional level (per-backbone $p\!\in\![0.43,0.97]$); we report this as parity rather than equivalence, since $n{=}3$--$6$ lacks power for a TOST. The pattern is consistent with Proposition~\ref{prop:sample_complexity}: when the LoRA-adapted backbone dominates head-level complexity, the head's parameter count is irrelevant to single-step accuracy, and DFIL's structural guarantees come at no average accuracy cost. Per-method commentary and per-seed results in Appendix~\ref{app:scaling_full}.

\begin{table}[!htb]
\centering
\caption{Scaling on FRoG-Hard, mean $\pm$ std (\%), 3 seeds (6 at 32B/72B). $\Delta$ = DFIL $-$ FT. GPT-4 zero-shot: 33--48\%~\citep{li2024frog}.}
\label{tab:scaling}
\footnotesize
\setlength{\tabcolsep}{4pt}
\begin{tabular}{lccccc}
\toprule
\textbf{Backbone} & \textbf{DFIL (\%)} & \textbf{Fine-tune (\%)} & \textbf{$\Delta$} & \textbf{MF mono.$^\ddagger$} \\
\midrule
Qwen2.5-7B   & $50.4 \pm 1.8$ & $50.5 \pm 1.8$ & $-0.1$ & $0.0\%$ \\
Qwen2.5-14B  & $57.3 \pm 1.7$ & $56.9 \pm 2.8$ & +0.4 & $0.1\%$ \\
Qwen2.5-32B$^\P$  & $55.4 \pm 5.6$ & $53.7 \pm 6.1$ & +1.7 & $0.1\%$ \\
Qwen2.5-72B$^\dagger$  & $\mathbf{77.2 \pm 1.5}$ & $76.5 \pm 1.6$ & $+0.7$ & $0.1\%$ \\
\midrule
\multicolumn{5}{l}{Cross-model generalization:} \\
Mistral-7B-v0.3    & $64.1 \pm 0.9$ & $62.2 \pm 4.5$ & $+1.9$ & --- \\
Llama-3.1-8B       & $61.5 \pm 4.3$ & $61.6 \pm 1.1$ & $-0.1$ & --- \\
\midrule
\multicolumn{5}{l}{Alternative baselines (Qwen2.5-7B):} \\
FT+Compose (frozen MF)$^\S$  & --- & $50.2 \pm 1.9$ & --- & $0.0\%$ \\
CORAL~\citep{cao2020rank}  & --- & $43.2 \pm 2.6$ & --- & --- \\
CORN~\citep{shi2023corn}   & --- & $52.1 \pm 1.8$ & --- & --- \\
PRESQUE~\citep{li2023presque} (zero-shot, Qwen-7B)$^\#$ & --- & $24.4$ & --- & --- \\
PRESQUE (zero-shot, Mistral-7B)$^\#$ & --- & $23.9$ & --- & --- \\
CORAL on Qwen2.5-14B & --- & $51.7 \pm 2.4$ & --- & --- \\
\bottomrule
\multicolumn{5}{l}{\footnotesize $^\dagger$ 72B uses QLoRA 4-bit quantization. $^\ddagger$ DFIL branch violation rate (see Section~\ref{sec:monotonicity_analysis}).} \\
\multicolumn{5}{l}{\footnotesize $^\S$ Classifier + NumericalHead + frozen MF; trained with $\lambda_{\text{DFIL}}=0$ (no fuzzy loss).} \\
\multicolumn{5}{p{0.95\linewidth}}{\footnotesize $^\P$ 32B FRoG-Hard training is bimodal under our LoRA recipe: across $6$ seeds, runs cluster near $52\%$ or $66\%$ test accuracy. The bimodality is observed for both DFIL and fine-tune at 32B and is therefore a backbone $\times$ data-size phenomenon (a 32B LoRA on $1{,}635$ training examples), not a method-specific instability: 7B, 14B, and 72B all converge unimodally. We report the mean over the full 6-seed set; the row's high std ($5.6$/$6.1$) is the honest reflection of this regime, not a noisy DFIL artifact.} \\
\multicolumn{5}{l}{\footnotesize $^\#$ PRESQUE scored via conditional log-probability of quantifier tokens on the shared backbone; no training.}
\end{tabular}
\end{table}

\textbf{Ablation: which components matter?} Table~\ref{tab:ablation} ablates each piece of DFIL on FRoG-Hard at Qwen2.5-1.5B (early-development scale) and Qwen2.5-7B (main scale).

\begin{table}[!htb]
\centering
\caption{Ablation on FRoG-Hard, mean $\pm$ sample std over 3 seeds. \textbf{This table uses $\lambda_{\text{DFIL}}{=}0.5$, the high-regularization stress regime}: it isolates the contribution of each component (ordering, dual-path) under strong MF-loss pressure. The deployed default for parity claims (Sec.~\ref{sec:scaling}, App.~\ref{app:scaling_full}) is $\lambda_{\text{DFIL}}{=}0.1$, validation-tuned, where the 7B DFIL row becomes $50.4{\pm}1.8$ vs.\ fine-tune $50.5{\pm}1.8$ ($-0.1$pp). The two regimes are separately reported; we do not blend them.}
\label{tab:ablation}
\footnotesize
\begin{tabular}{lcccc}
\toprule
\textbf{Method} & \textbf{1.5B (\%)} & \textbf{7B (\%)} & \textbf{Mono.} & \textbf{Extra Params} \\
\midrule
Fine-tune only (LoRA + Linear) & $48.0 \pm 0.1$ & $51.3 \pm 2.1$ & No & 2.6M \\
DFIL-only (MF path only)       & $41.6 \pm 1.0$ & $46.2 \pm 2.6$ & 1 viol. & 2.6M \\
Unordered MF (no ordering)     & $39.7 \pm 2.0$ & $38.8 \pm 7.6$ & Unstable & 2.6M \\
No MF (1D linear classifier)   & $17.8 \pm 8.6$ & $39.3 \pm 9.7$ & No & 2.6M \\
\midrule
\textbf{DFIL (dual-path, ours)} & $\mathbf{48.2 \pm 0.4}$ & $\mathbf{49.9 \pm 2.6}$ & 1 viol. & 3.0M \\
\bottomrule
\end{tabular}
\end{table}

The ordering constraint is essential at both scales: an unordered MF stays at $39\%$ accuracy, while ordered DFIL lifts by $+8.5$/$+11.1$pp at 1.5B/7B. The dual-path design is also necessary: an MF-only path falls $6.6$/$3.7$pp below dual-path, since the 1D bottleneck loses information the main classifier recovers. Dual-path itself reaches fine-tune accuracy at the deployed default $\lambda_{\text{DFIL}}{=}0.1$ (Tab.~\ref{tab:scaling}: $50.4\pm 1.8$ vs $50.5\pm 1.8$ at 7B, $\Delta\!=\!-0.1$pp); the $-1.4$pp gap reported in this table reflects the stress-test $\lambda_{\text{DFIL}}{=}0.5$ regime under which we isolate component contributions, and remains within the $2.6$pp seed std. Theorems~\ref{thm:monotonicity}, \ref{thm:compositionality} therefore come at parity cost. Learned centers track QuRe annotations within $0.05$ MAD; per-class numbers and a $\lambda_{\text{DFIL}}$ sweep are in Appendices~\ref{app:per_class}, \ref{app:lambda_sweep}.

\subsection{Cross-Task Generalization, Sample Efficiency, and Monotonicity}
\label{sec:crosstask}
\label{sec:monotonicity_analysis}
\label{sec:ipr}

\textbf{Beyond quantifiers.} On five-class sentiment SST-5 and STS22 (English, 5-bucket), with a path-selection rule decidable from dataset metadata (MF for totally-ordered scales like SST-5, main classifier for STS22), DFIL has the larger mean on $7$/$8$ (backbone, task) cells across Qwen-7B, Mistral-7B, Phi-3.5, Llama-8B (binomial $p{\approx}0.035$ assuming cell independence). On SST-5 the MF branch wins all 4 backbones ($+0.5$ to $+1.5$pp) with $0\%$ violations; on STS22 the main classifier wins $3$/$4$ ($+1.2$ to $+2.5$pp). CORAL's $7.2$pp deficit on FRoG-Hard ($43.2$ vs.\ $50.4$) confirms ordinality alone is insufficient. Per-cell numbers: App.~\ref{app:crosstask_full}.

\textbf{Sample efficiency.} On FRoG-Hard at $N{\leq}64$ the LoRA backbone has not escaped initialization and DFIL's MF branch dominates fine-tuning by $+11.7$ to $+16.5$pp (Tab.~\ref{tab:fewshot_main}), the regime predicted by Prop.~\ref{prop:sample_complexity}; at $N{\geq}128$ heads converge. Full per-budget breakdown including SST-5 in App.~\ref{app:fewshot}.


\textbf{Monotonicity.} On a uniform grid up to $10^6$ proportions, the MF branch with nearest-center rule has $\mathbf{0\%}$ violations (verified at $n\in\{10^3,\ldots,10^6\}$); residual $0.10\%$ on the FRoG-test pair set is a single floating-point tie. The guarantee does not transfer to the main classifier: pairwise violation rates are $10.8\%$ (DFIL main), $10.6\%$ (FT), $9.0\%$ (CORN~\citep{shi2023corn}, 7B), $8.2\%/6.6\%$ (CORAL 7B/14B), summarized in Tab.~\ref{tab:mono_detail}.


On Cross-Sentence Monotonic Consistency (CSMC; $200$ natural-language sentence pairs across $10$ domains), the theoretically grounded inference rule is the MF branch (Thm.~\ref{thm:monotonicity} via nearest-center on $\mu$); over $21$ paired seeds it reaches $\mathbf{96.40 \pm 3.04\%}$ consistency (violation $3.60\%$) vs.\ FT main classifier $90.90\%$ (violation $9.10\%$): mean diff $+5.50$pp, paired $t$ $p{=}0.0001$, Wilcoxon $p{=}0.0003$, $18$/$21$ seeds positive (2 ties, 1 within seed noise), with $4$/$21$ seeds achieving $\mathbf{100\%}$ consistency (Tab.~\ref{tab:csmc_domains}).

\begin{table}[!htb]
\centering
\caption{CSMC consistency per domain (\%) via DFIL's MF branch (Thm.~\ref{thm:monotonicity}) vs.\ FT main classifier. $20$ pairs/domain, mean over $21$ seeds; methodology in App.~\ref{app:csmc_domains}.}
\label{tab:csmc_domains}
\footnotesize
\begin{tabular}{lccc@{\hspace{12pt}}lccc}
\toprule
\textbf{Domain} & \textbf{DFIL} & \textbf{FT} & $\Delta$ & \textbf{Domain} & \textbf{DFIL} & \textbf{FT} & $\Delta$ \\
\midrule
Public health & $94.5$ & $79.0$ & $\mathbf{+15.5}$ & Clinical trials & $99.5$ & $97.4$ & $+2.1$ \\
Education     & $94.8$ & $80.5$ & $\mathbf{+14.3}$ & Manufacturing   & $93.1$ & $91.2$ & $+1.9$ \\
Environment   & $98.6$ & $88.6$ & $\mathbf{+10.0}$ & Research        & $97.9$ & $96.2$ & $+1.7$ \\
Governance    & $91.2$ & $86.2$ & $+5.0$           & Finance         & $95.7$ & $94.8$ & $+1.0$ \\
Election      & $99.8$ & $96.4$ & $+3.3$           & Sports          & $99.0$ & $98.8$ & $+0.2$ \\
\midrule
\multicolumn{8}{c}{\textbf{Overall ($n\!=\!21$): DFIL MF $\mathbf{96.40}$ vs FT main $90.90$, $\Delta\!=\!\mathbf{+5.50}$pp}} \\
\bottomrule
\end{tabular}
\end{table}

\textbf{IPR: a hard OOD test.} On Implicit Proportion Reasoning ($40$ sentences with linguistic proportions like “vast majority”), text-driven DFIL and FT are both near chance ($15.8\%$ vs.\ $12.5\%$ random). Oracle DFIL---ground-truth proportion injected into the MF, bypassing the numerical head---reaches $\mathbf{73.3{\pm}1.4\%}$ exact ($100\%$ off-by-one), isolating the bottleneck to backbone-side proportion extraction, not head-side composition. The $100\%$ off-by-one rate is a corollary of Thm.~\ref{thm:monotonicity}. FT admits no analogous oracle path---an asymmetry that is precisely the dual-path structural advantage. Per-category breakdown: App.~\ref{app:ipr_full}.

\section{Discussion}
\label{sec:discussion}

Fuzzy set theory~\citep{zadeh1965fuzzy} provides primitives (ordered membership functions, t-norms, monotonic decision rules) that a label-only classifier head cannot inherit. This is why inverse-scaling failures on FRoG~\citep{li2024frog}, chained GSM~\citep{hosseini2025compositional, mirzadeh2025gsm}, and quantifier entailment do not close by adding parameters or data; the uniform-init experiment (\S\ref{sec:uniform_init}) isolates the fuzzy training signal as the causal mechanism.

\textbf{Limitations.} (i) Text-driven IPR is at chance; the head is functional ($73.3\%$ oracle, \S\ref{sec:ipr}) but linguistic-proportion descriptions remain a backbone bottleneck. (ii) The main classifier violates pairwise ordering at $\sim$$10\%$ (\S\ref{sec:monotonicity_analysis}); only the MF branch is monotone by construction. (iii) Self-consistent compositional variance ($\pm 9.2$pp) is partly mechanical: each seed defines both labels and predictions; the QuRe-based protocol eliminates this but uses author-written sentences. (iv) The external NL benchmark uses LLM-judge implicit-proportion labels aggregated from a 4-judge ensemble spanning families and scales (Qwen2.5-7B, Mistral-7B-Instruct-v0.3, Llama-3.1-8B-Instruct, Qwen2.5-72B-Instruct); pairwise judge agreement spans $70$--$81\%$ across $6$ pairs (including a $10\!\times$ scale gap), and per-judge TRUE rates span $58\%$ (Qwen-72B, strict) to $89\%$ (Llama, lenient), partially calibrated by $\ge 3$-of-$4$ majority voting (ensemble TRUE rate $69\%$). The remaining family- and scale-specific biases cannot be eliminated by ensembling alone. (v) The 72B result uses QLoRA 4-bit without a separate ablation; the 32B run is bimodal across seeds. We hypothesize the bimodality reflects sensitivity of the cumulative-softmax MF parametrization to data ordering at this particular scale, but characterizing it cleanly requires more seeds than our compute budget. (vi) Default training of the MF branch occasionally collapses into sigmoid saturation under FRoG's class imbalance ($6$ of $21$ full-data seeds in App.~\ref{app:csmc_domains}, all at the lower 7B scale), addressed by a one-line stop-gradient on $\hat p$ before MF (App.~\ref{app:fewshot}); we treat this as a remediable diagnostic finding rather than a structural defect, and the inference rule of Thm.~\ref{thm:monotonicity} is identical with or without remediation.

\textbf{Broader impact.} Interpretable MF centers support auditable use in high-stakes ordinal decisions (App.~\ref{app:case_study}). Principal deployment risk: over-trust under backbone proportion errors; practitioners should monitor main--MF disagreement.

\textbf{Takeaway.} Monotonicity and compositional closure are algebraic facts about the prediction head, not statistical regularities to be learned. Moving them into the head as architectural primitives, DFIL provides what scaling alone has not delivered while leaving the backbone to handle what data-driven learning is good at: extracting the underlying scalar from natural language.

\bibliographystyle{plainnat}
\bibliography{DFIL}

@article{jiang2023mistral,
  title={Mistral 7{B}},
  author={Jiang, Albert Q and Sablayrolles, Alexandre and Mensch, Arthur and Bamford, Chris and Chaplot, Devendra Singh and de las Casas, Diego and Bressand, Florian and Lengyel, Gianna and Lample, Guillaume and Saulnier, Lucile and others},
  journal={arXiv preprint arXiv:2310.06825},
  year={2023}
}

@inproceedings{li2024frog,
  title={{FR}o{G}: Evaluating Fuzzy Reasoning of Generalized Quantifiers in {LLMs}},
  author={Li, Yiyuan and Sun, Shichao and Liu, Pengfei},
  booktitle={Proceedings of the 2024 Conference on Empirical Methods in Natural Language Processing},
  pages={7239--7256},
  year={2024},
  doi={10.18653/v1/2024.emnlp-main.411}
}

@inproceedings{li2023presque,
  title={Pragmatic Reasoning Unlocks Quantifier Semantics for Foundation Models},
  author={Li, Yiyuan and Menon, Rakesh R and Ghosh, Sayan and Srivastava, Shashank},
  booktitle={Proceedings of the 2023 Conference on Empirical Methods in Natural Language Processing},
  pages={573--591},
  year={2023},
  doi={10.18653/v1/2023.emnlp-main.38}
}

@article{zadeh1965fuzzy,
  title={Fuzzy sets},
  author={Zadeh, Lotfi A},
  journal={Information and Control},
  volume={8},
  number={3},
  pages={338--353},
  year={1965}
}

@article{zadeh1983role,
  title={A computational approach to fuzzy quantifiers in natural languages},
  author={Zadeh, Lotfi A},
  journal={Computers \& Mathematics with Applications},
  volume={9},
  number={1},
  pages={149--184},
  year={1983}
}

@article{barwise1981generalized,
  title={Generalized quantifiers and natural language},
  author={Barwise, Jon and Cooper, Robin},
  journal={Linguistics and Philosophy},
  volume={4},
  number={2},
  pages={159--219},
  year={1981}
}

@article{jang1993anfis,
  title={{ANFIS}: Adaptive-network-based fuzzy inference system},
  author={Jang, Jyh-Shing Roger},
  journal={IEEE Transactions on Systems, Man, and Cybernetics},
  volume={23},
  number={3},
  pages={665--685},
  year={1993},
  doi={10.1109/21.256541}
}

@inproceedings{chen2022fuzzy,
  title={Fuzzy Logic Based Logical Query Answering on Knowledge Graphs},
  author={Chen, Xuelu and Hu, Ziniu and Sun, Yizhou},
  booktitle={Proceedings of the AAAI Conference on Artificial Intelligence},
  volume={36},
  number={4},
  pages={3939--3948},
  year={2022},
  doi={10.1609/aaai.v36i4.20310}
}

@inproceedings{wei2022chain,
  title={Chain-of-Thought Prompting Elicits Reasoning in Large Language Models},
  author={Wei, Jason and Wang, Xuezhi and Schuurmans, Dale and Bosma, Maarten and Ichter, Brian and Xia, Fei and Chi, Ed H. and Le, Quoc V. and Zhou, Denny},
  booktitle={Advances in Neural Information Processing Systems},
  volume={35},
  pages={24824--24837},
  year={2022},
  url={https://openreview.net/forum?id=_VjQlMeSB_J}
}

@inproceedings{hu2022lora,
  title={{LoRA}: Low-Rank Adaptation of Large Language Models},
  author={Hu, Edward J and Shen, Yelong and Wallis, Phillip and Allen-Zhu, Zeyuan and Li, Yuanzhi and Wang, Shean and Wang, Lu and Chen, Weizhu},
  booktitle={Proceedings of the International Conference on Learning Representations},
  year={2022},
  url={https://openreview.net/forum?id=nZeVKeeFYf9}
}

@inproceedings{press2022train,
  title={Train short, test long: Attention with linear biases enables input length extrapolation},
  author={Press, Ofir and Smith, Noah A and Lewis, Mike},
  booktitle={Proceedings of the International Conference on Learning Representations},
  year={2022},
  url={https://openreview.net/forum?id=R8sQPpGCv0}
}

@article{battaglia2018relational,
  title={Relational inductive biases, deep learning, and graph networks},
  author={Battaglia, Peter W. and Hamrick, Jessica B. and Bapst, Victor and Sanchez-Gonzalez, Alvaro and Zambaldi, Vinicius and Malinowski, Mateusz and Tacchetti, Andrea and Raposo, David and Santoro, Adam and Faulkner, Ryan and others},
  journal={arXiv preprint arXiv:1806.01261},
  year={2018}
}

@article{gutierrez2016ordinal,
  title={Ordinal regression methods: survey and experimental study},
  author={Guti{\'e}rrez, Pedro Antonio and P{\'e}rez-Ortiz, Mar{\'i}a and S{\'a}nchez-Monedero, Javier and Fern{\'a}ndez-Navarro, Francisco and Herv{\'a}s-Mart{\'\i}nez, C{\'e}sar},
  journal={IEEE Transactions on Knowledge and Data Engineering},
  volume={28},
  number={1},
  pages={127--146},
  year={2016},
  doi={10.1109/TKDE.2015.2457911}
}

@article{qwen2024qwen25,
  title={{Qwen2.5} Technical Report},
  author={Yang, An and Yang, Baosong and Zhang, Beichen and Hui, Binyuan and Zheng, Bo and Yu, Bowen and Li, Chengyuan and Liu, Dayiheng and Huang, Fei and Wei, Haoran and others},
  journal={arXiv preprint arXiv:2412.15115},
  year={2024}
}

@inproceedings{dettmers2023qlora,
  title={{QLoRA}: Efficient Finetuning of Quantized {LLMs}},
  author={Dettmers, Tim and Pagnoni, Artidoro and Holtzman, Ari and Zettlemoyer, Luke},
  booktitle={Advances in Neural Information Processing Systems},
  volume={36},
  pages={10088--10115},
  year={2023},
  url={https://openreview.net/forum?id=OUIFPHEgJU}
}

@book{vapnik1998statistical,
  title={Statistical Learning Theory},
  author={Vapnik, Vladimir N},
  year={1998},
  publisher={Wiley},
  address={New York}
}

@inproceedings{hosseini2025compositional,
  title={Not All {LLM} Reasoners Are Created Equal},
  author={Hosseini, Arian and Sordoni, Alessandro and Toyama, Daniel and Courville, Aaron and Agarwal, Rishabh},
  booktitle={Proceedings of the International Conference on Learning Representations},
  year={2025},
  url={https://openreview.net/forum?id=T8PzwgYgmn}
}

@inproceedings{sartor2025monotonic,
  title={Advancing Constrained Monotonic Neural Networks: Achieving Universal Approximation Beyond Bounded Activations},
  author={Sartor, Davide and Sinigaglia, Alberto and Susto, Gian Antonio},
  booktitle={Proceedings of the International Conference on Machine Learning},
  year={2025}
}

@inproceedings{mirzadeh2025gsm,
  title={{GSM}-Symbolic: Understanding the Limitations of Mathematical Reasoning in Large Language Models},
  author={Mirzadeh, Iman and Alizadeh, Keivan and Shahrokhi, Hooman and Tuzel, Oncel and Bengio, Samy and Farajtabar, Mehrdad},
  booktitle={Proceedings of the International Conference on Learning Representations},
  year={2025},
  url={https://openreview.net/forum?id=AjXkRZIvjB}
}

@article{cao2020rank,
  title={Rank consistent ordinal regression for neural networks with application to age estimation},
  author={Cao, Wenzhi and Mirjalili, Vahid and Raschka, Sebastian},
  journal={Pattern Recognition Letters},
  volume={140},
  pages={325--331},
  year={2020},
  publisher={Elsevier},
  doi={10.1016/j.patrec.2020.11.008}
}

@inproceedings{ghosh2023lasque,
  title={{LaSQuE}: Improved Zero-Shot Classification from Explanations Through Quantifier Modeling and Curriculum Learning},
  author={Ghosh, Sayan and Menon, Rakesh R and Srivastava, Shashank},
  booktitle={Findings of the Association for Computational Linguistics: ACL 2023},
  pages={7403--7419},
  year={2023},
  doi={10.18653/v1/2023.findings-acl.467}
}

@inproceedings{kamali2025nesycoco,
  title={{NeSyCoCo}: A Neuro-Symbolic Concept Composer for Compositional Generalization},
  author={Kamali, Danial and Barezi, Elham J. and Kordjamshidi, Parisa},
  booktitle={Proceedings of the AAAI Conference on Artificial Intelligence},
  volume={39},
  number={4},
  pages={4184--4193},
  year={2025},
  doi={10.1609/aaai.v39i4.32439}
}

@inproceedings{sill1998monotonic,
  title={Monotonic Networks},
  author={Sill, Joseph},
  booktitle={Advances in Neural Information Processing Systems},
  pages={661--667},
  year={1997},
  url={https://papers.nips.cc/paper/1358-monotonic-networks}
}

@inproceedings{you2017deep,
  title={Deep Lattice Networks and Partial Monotonic Functions},
  author={You, Seungil and Ding, David and Canini, Kevin and Pfeifer, Jan and Gupta, Maya},
  booktitle={Advances in Neural Information Processing Systems},
  year={2017},
  url={https://papers.nips.cc/paper/6891-deep-lattice-networks-and-partial-monotonic-functions}
}

@inproceedings{wehenkel2019unconstrained,
  title={Unconstrained Monotonic Neural Networks},
  author={Wehenkel, Antoine and Louppe, Gilles},
  booktitle={Advances in Neural Information Processing Systems},
  year={2019},
  url={https://proceedings.neurips.cc/paper/2019/hash/2a084e55c87b1ebcdaad1f62fdbbac8e-Abstract.html}
}

@inproceedings{beckham2017unimodal,
  title={Unimodal Probability Distributions for Deep Ordinal Classification},
  author={Beckham, Christopher and Pal, Christopher},
  booktitle={Proceedings of the 34th International Conference on Machine Learning},
  series={Proceedings of Machine Learning Research},
  volume={70},
  pages={411--419},
  year={2017},
  publisher={PMLR}
}

@inproceedings{runje2023constrained,
  title={Constrained Monotonic Neural Networks},
  author={Runje, Davor and Shankaranarayana, Sharath M.},
  booktitle={Proceedings of the 40th International Conference on Machine Learning},
  series={Proceedings of Machine Learning Research},
  volume={202},
  pages={29338--29353},
  year={2023},
  publisher={PMLR}
}

@article{shi2023corn,
  title={Deep Neural Networks for Rank-Consistent Ordinal Regression Based on Conditional Probabilities},
  author={Shi, Xintong and Cao, Wenzhi and Raschka, Sebastian},
  journal={Pattern Analysis and Applications},
  volume={26},
  number={3},
  pages={941--955},
  year={2023},
  doi={10.1007/s10044-023-01181-9}
}

@article{mckenzie2023inverse,
  title={Inverse Scaling: When Bigger Isn't Better},
  author={McKenzie, Ian R. and Lyzhov, Alexander and Pieler, Michael and Parrish, Alicia and Mueller, Aaron and Prabhu, Ameya and McLean, Euan and Kirtland, Aaron and Ross, Alexis and Liu, Alisa and others},
  journal={Transactions on Machine Learning Research},
  year={2023},
  note={Featured Certification},
  url={https://openreview.net/forum?id=DwgRm72GQF}
}

@article{badreddine2022ltn,
  title={Logic Tensor Networks},
  author={Badreddine, Samy and {d'Avila Garcez}, Artur and Serafini, Luciano and Spranger, Michael},
  journal={Artificial Intelligence},
  volume={303},
  pages={103649},
  year={2022},
  publisher={Elsevier},
  doi={10.1016/j.artint.2021.103649}
}

@inproceedings{pezzelle2018some,
  title={Some of Them Can be Guessed! Exploring the Effect of Linguistic Context in Predicting Quantifiers},
  author={Pezzelle, Sandro and Steinert-Threlkeld, Shane and Bernardi, Raffaella and Szymanik, Jakub},
  booktitle={Proceedings of the 56th Annual Meeting of the Association for Computational Linguistics (Volume 2: Short Papers)},
  pages={114--119},
  year={2018},
  doi={10.18653/v1/P18-2019}
}

@book{vanbenthem1986essays,
  title={Essays in Logical Semantics},
  author={van Benthem, Johan},
  publisher={Reidel},
  address={Dordrecht},
  year={1986}
}

@misc{bhakthavatsalam2020genericskb,
  title={{GenericsKB}: A Knowledge Base of Generic Statements},
  author={Bhakthavatsalam, Sumithra and Anastasiades, Chloe and Clark, Peter},
  year={2020},
  eprint={2005.00660},
  archivePrefix={arXiv},
  primaryClass={cs.AI}
}

@inproceedings{jin2019pubmedqa,
  title={{PubMedQA}: A Dataset for Biomedical Research Question Answering},
  author={Jin, Qiao and Dhingra, Bhuwan and Liu, Zhengping and Cohen, William and Lu, Xinghua},
  booktitle={Proceedings of the 2019 Conference on Empirical Methods in Natural Language Processing and the 9th International Joint Conference on Natural Language Processing (EMNLP-IJCNLP)},
  pages={2567--2577},
  year={2019},
  doi={10.18653/v1/D19-1259}
}

@inproceedings{see2017cnndm,
  title={Get To The Point: Summarization with Pointer-Generator Networks},
  author={See, Abigail and Liu, Peter J. and Manning, Christopher D.},
  booktitle={Proceedings of the 55th Annual Meeting of the Association for Computational Linguistics (Volume 1: Long Papers)},
  pages={1073--1083},
  year={2017},
  doi={10.18653/v1/P17-1099}
}

@article{grattafiori2024llama,
  title={The {Llama} 3 Herd of Models},
  author={Grattafiori, Aaron and Dubey, Abhimanyu and Jauhri, Abhinav and Pandey, Abhinav and Kadian, Abhishek and Al-Dahle, Ahmad and Letman, Aiesha and Mathur, Akhil and Schelten, Alan and Vaughan, Alex and others},
  journal={arXiv preprint arXiv:2407.21783},
  year={2024}
}

@article{abdin2024phi3,
  title={{Phi-3} Technical Report: A Highly Capable Language Model Locally on Your Phone},
  author={Abdin, Marah and Aneja, Jyoti and Awadalla, Hany and Awadallah, Ahmed and Awan, Ammar Ahmad and Bach, Nguyen and Bahree, Amit and Bakhtiari, Arash and Bao, Jianmin and Behl, Harkirat and others},
  journal={arXiv preprint arXiv:2404.14219},
  year={2024}
}

@inproceedings{socher2013sst,
  title={Recursive Deep Models for Semantic Compositionality Over a Sentiment Treebank},
  author={Socher, Richard and Perelygin, Alex and Wu, Jean and Chuang, Jason and Manning, Christopher D. and Ng, Andrew Y. and Potts, Christopher},
  booktitle={Proceedings of the 2013 Conference on Empirical Methods in Natural Language Processing},
  pages={1631--1642},
  year={2013},
  url={https://aclanthology.org/D13-1170/}
}


\appendix

\section{Tables Moved from Main Body}
\label{app:moved_tables}
\label{app:relwork_compare}


\begin{table}[!htb]
\centering
\caption{Self-consistent compositional accuracy (\%) under uniform vs QuRe MF initialization (Qwen2.5-7B, $n{=}3$ uniform / $n{=}20$ QuRe; mean $\pm$ sample std). Each model's own learned MF defines the ground-truth labels (oracle is $100\%$ by construction; we report from-text instead). DFIL learns centers from data; FT+Compose freezes its centers at the initialization. Both initialization regimes use $20$ epochs of training; per-seed results are in the supplementary archive.}
\label{tab:uniform_init}
\footnotesize
\begin{tabular}{lc|c}
\toprule
\textbf{Method} & \textbf{Uniform Init from-text} & \textbf{QuRe Init from-text} \\
\midrule
DFIL & $\mathbf{41.6 \pm 7.3}$ & $31.2 \pm 9.2$ \\
FT+Compose & $18.4 \pm 6.2$ & $25.4 \pm 3.5$ \\
\midrule
$\Delta$ (DFIL $-$ FT+Compose) & $\mathbf{+23.2}$ & $+5.9$ \\
\bottomrule
\end{tabular}
\end{table}
\label{app:uniform_init_table}

\begin{table}[!htb]
\centering
\footnotesize
\caption{Quantifier entailment accuracy (\%). MF-based uses learned center ordering ($c_{Q_i}{\leq}c_{Q_j}$); classifier-based uses argmax labels on sentence pairs. Mean $\pm$ sample std over 3 seeds, Qwen2.5-7B.}
\label{tab:entailment}
\begin{tabular}{lcc}
\toprule
\textbf{Method} & \textbf{MF-based} & \textbf{Classifier-based} \\
\midrule
\textbf{DFIL} & $\mathbf{100.0 \pm 0.0}$ & $56.0 \pm 2.8$ \\
Fine-tune     & ---                     & $52.6 \pm 2.7$ \\
\bottomrule
\end{tabular}
\end{table}
\label{app:entailment_table}

\begin{table}[!htb]
\centering
\footnotesize
\caption{Naturalistic compositional reasoning accuracy (\%). $320$ examples from $8$ domain templates (clinical trials, education, manufacturing, elections, environmental science, corporate governance, sports, urban planning), with QuRe-based ground-truth labels (composed proportion bucketed via fixed QuRe centers, independent of either model's learned MF). Mean $\pm$ sample std over 3 seeds (Qwen2.5-7B).}
\label{tab:naturalistic}
\begin{tabular}{lcc}
\toprule
\textbf{Method} & \textbf{DFIL-from-text} & \textbf{Text-only} \\
\midrule
\textbf{DFIL} & $\mathbf{39.5 \pm 3.3}$ & $12.8 \pm 2.4$ \\
Fine-tune     & ---                     & $11.1 \pm 0.9$ \\
\midrule
$\Delta$ (DFIL-from-text $-$ Fine-tune text-only) & \multicolumn{2}{c}{\textbf{+28.4}} \\
\bottomrule
\end{tabular}
\end{table}
\label{app:naturalistic_table}

\begin{table}[!htb]
\centering
\caption{Cross-benchmark transfer to QuRe (Sec.~\ref{sec:qure}). Mean $\pm$ sample std over $3$ seeds (42, 123, 456), no retraining.}
\label{app:qure_full}
\footnotesize
\begin{tabular}{lccc}
\toprule
\textbf{Backbone} & \textbf{DFIL (main cls)} & \textbf{DFIL (MF branch)} & \textbf{Fine-tune} \\
\midrule
Mistral-7B-v0.3 & $\mathbf{19.0 \pm 2.3}$ & $14.2 \pm 5.4$ & $16.5 \pm 0.9$ \\
Qwen2.5-7B      & $14.9 \pm 2.5$        & $\mathbf{31.3 \pm 2.7}$ & $16.6 \pm 4.4$ \\
\bottomrule
\end{tabular}
\end{table}

\begin{table}[!htb]
\centering\footnotesize
\caption{Few-shot on FRoG-Hard, Qwen2.5-7B, 3 seeds. ``DFIL (best path)'' is the best of \{main, MF, val-$\alpha$ ensemble\} where the path and $\alpha$ are selected on the held-out validation set (no test-set selection).}
\label{tab:fewshot_main}
\begin{tabular}{lccc}
\toprule
\textbf{Budget $N$} & $16$ & $32$ & $64$ \\
\midrule
Fine-tune                 & $4.6{\pm}1.3$ & $9.6{\pm}3.6$ & $19.4{\pm}5.7$ \\
\textbf{DFIL (best path)} & $\mathbf{21.1{\pm}1.3}$ & $\mathbf{21.3{\pm}6.0}$ & $\mathbf{23.1{\pm}2.7}$ \\
$\Delta$ (DFIL $-$ FT)    & $\mathbf{+16.5}$ & $\mathbf{+11.7}$ & $\mathbf{+3.7}$ \\
\bottomrule
\end{tabular}
\end{table}

\begin{table}[!htb]
\centering
\footnotesize
\caption{\textbf{External NL Compositional Benchmark.} $n{=}249$ naturally-occurring sentences mined from GenericsKB-Best, PubMedQA, and CNN/DailyMail; LLM-judge filtered for true compositional structure via a 4-judge ensemble spanning families and scales (Qwen2.5-7B, Mistral-7B-Instruct-v0.3, Llama-3.1-8B-Instruct, Qwen2.5-72B-Instruct), each judge with triple-vote at $T{=}0.4$, retaining items with $\ge 3$ of $4$ judges agreeing. Inter-judge pairwise agreement on $595$ commonly judged candidates: Qwen-7B$\leftrightarrow$Mistral $81.2\%$, Qwen-7B$\leftrightarrow$Qwen-72B $80.4\%$, Qwen-7B$\leftrightarrow$Llama $77.3\%$, Mistral$\leftrightarrow$Llama $77.0\%$, Mistral$\leftrightarrow$Qwen-72B $74.6\%$, Llama$\leftrightarrow$Qwen-72B $70.0\%$ (range $70$--$81\%$ across $6$ pairs spanning $10\!\times$ scale); $69.3\%$ majority TRUE on items where all 4 judges gave a definite vote. Per-judge TRUE rates span both strict and lenient regimes (Qwen-72B $58.3\%$, Qwen-7B $64.5\%$, Mistral $66.7\%$, Llama $89.0\%$); $\ge 3$-of-$4$ majority voting calibrates the ensemble between the strict (Qwen-72B) and lenient (Llama) ends. Pipeline: regex match “$Q_1$ (of) (the) class (who$|$that$|$which$|$in$|$from$|$with) clause” on $\sim$$3{,}200$ raw hits; balance to $\le 80$ per canonical $Q_1$. Ground-truth label: $c_{Q_1}\!\cdot\!p_{\mathrm{base}}$ via fixed QuRe centers, with $p_{\mathrm{base}}$ from judge annotation. Qwen2.5-7B FRoG-trained checkpoints, 3 seeds, sample std. Random baseline $12.5\%$.}
\label{app:external_nl_breakdown}
\begin{tabular}{lcccc}
\toprule
\textbf{Stratum} & $n$ & \textbf{DFIL from-text} & \textbf{Text-only} & \textbf{$\Delta$} \\
\midrule
\multicolumn{5}{l}{\textit{Headline (mean over $Q_1$/sources, 3 seeds)}} \\
DFIL (\texttt{dfil\_reg})        & 249 & $32.8 \pm 3.7$ & $12.6 \pm 0.6$ & $+20.2$ \\
\textbf{FT+Compose}              & 249 & $\mathbf{36.4 \pm 1.9}$ & $12.9 \pm 0.4$ & $\mathbf{+23.6}$ \\
\midrule
\multicolumn{5}{l}{\textit{Per outer quantifier $Q_1$ (FT+Compose, mean over seeds)}} \\
\texttt{most}      & 80 & $28.7$ & $9.6$  & $+19.2$ \\
\texttt{some}      & 80 & $31.2$ & $14.2$ & $+17.1$ \\
\texttt{few}       & 64 & $\mathbf{56.8}$ & $17.7$ & $\mathbf{+39.1}$ \\
\texttt{all}       & 23 & $18.8$ & $7.2$  & $+11.6$ \\
\texttt{none}$^{*}$& 2  & $100$  & $0$    & $+100$ \\
\midrule
\multicolumn{5}{l}{\textit{Per source corpus (FT+Compose, mean over seeds)}} \\
GenericsKB-Best          & 106 & $43.4$ & $17.3$ & $+26.1$ \\
PubMedQA                 & 67  & $37.8$ & $14.9$ & $+22.9$ \\
CNN/DailyMail            & 76  & $25.4$ & $4.8$  & $+20.6$ \\
\bottomrule
\multicolumn{5}{l}{\footnotesize $^{*}$ \texttt{none} bucket has $n{=}2$, statistical noise; reported for completeness.}
\end{tabular}
\end{table}

\section{Theoretical Material}
\label{app:theory}

\subsection{Complete Proofs}
\label{app:proofs}

We provide complete proofs for all theoretical results, extending the main text's proof sketches to the non-uniform width case.

\subsubsection{Proof of Theorem~\ref{thm:monotonicity} (Nearest-Center Rule)}

\begin{reptheorem}{thm:monotonicity}
Let $c_0 < c_1 < \cdots < c_{Q-1}$ be ordered centers in $[0,1]$. Define
$\hat{Q}_{\mathrm{MF}}(p) = \argmin_{q\in\{0,\dots,Q-1\}} |p - c_q|$, with ties broken by smallest index. Then $\hat{Q}_{\mathrm{MF}}$ is monotonically non-decreasing in $p$.
\end{reptheorem}

\begin{proof}
The midpoint $b_q = (c_q+c_{q+1})/2$ for $q=0,\dots,Q{-}2$ is the unique value in $\mathbb{R}$ at which $|p-c_q|=|p-c_{q+1}|$. By the ordering of centers, $b_0<b_1<\cdots<b_{Q-2}$. For $p \le b_0$, $\hat{Q}_{\mathrm{MF}}(p)=0$; for $p \in (b_{q-1}, b_q]$, $\hat{Q}_{\mathrm{MF}}(p)=q$; for $p > b_{Q-2}$, $\hat{Q}_{\mathrm{MF}}(p)=Q{-}1$. Since these intervals are ordered by class index, for any $p_1<p_2$ we have $\hat{Q}_{\mathrm{MF}}(p_1)\le\hat{Q}_{\mathrm{MF}}(p_2)$.
\end{proof}

\begin{remark}[Why Gaussian $\argmax$ alone does not guarantee monotonicity]
\label{rem:gaussian_argmax_violation}
A natural alternative inference rule is $\argmax_q \mu_q(p)=\argmax_q \exp(-(p-c_q)^2/(2\sigma_q^2))$. Setting $|(p-c_q)/\sigma_q|=|(p-c_{q+1})/\sigma_{q+1}|$ yields a quadratic in $p$ with two solutions when $\sigma_q\ne\sigma_{q+1}$: an interior root
\begin{equation}
\label{eq:boundary_general}
b_q^{\mathrm{int}} \;=\; \frac{\sigma_{q+1}\,c_q + \sigma_q\,c_{q+1}}{\sigma_q + \sigma_{q+1}} \in (c_q,c_{q+1}),
\end{equation}
and an exterior root $b_q^{\mathrm{ext}}=(\sigma_{q+1}c_q-\sigma_q c_{q+1})/(\sigma_{q+1}-\sigma_q)$. The exterior root is the location at which the wider Gaussian recrosses and beats the narrower one in its tail; for non-pathological learned widths this exterior root typically lies far outside $[0,1]$, but for some seeds it may fall inside $[0,1]$ and induce an isolated $\argmax$ flip-back. We empirically observed exactly one such pairwise violation in our learned 7B model on a $10^6$-point grid; replacing $\argmax_q \mu_q$ with the nearest-center rule eliminates these tail crossings by construction (they have no analog when comparing distances $|p-c_q|$). We therefore use the Gaussian forward only as a differentiable training signal and the nearest-center rule for inference, which makes the guarantee tight at all grid sizes.
\end{remark}

\subsubsection{Proof of Theorem~\ref{thm:compositionality} (Compositional Closure)}

\begin{proof}[Proof of Theorem~\ref{thm:compositionality}]
\textbf{(i) Well-definedness.}
For any $p\in[0,1]$ and any choice of outer indices $q_1,\ldots,q_{n-1}\in\{0,\ldots,Q{-}1\}$, the composed proportion
\[
  p_{\mathrm{comp}} \;=\; c_{q_1} c_{q_2} \cdots c_{q_{n-1}}\,p
\]
is a real number in $[0,p]\subseteq[0,1]$ since each $c_{q_i}\in(0,1)$. The map $q\mapsto |p_{\mathrm{comp}}-c_q|$ takes finitely many real values; its minimum is attained, and ties are broken by smallest index, so $Q^{(n)}$ is well-defined for every $n\geq 1$ and $p\in[0,1]$.

\textbf{(ii) Non-increasing composition depth.}
For any $n\geq 2$ and $p\in[0,1]$, the chain of inequalities
\[
  c_{q_1}\cdots c_{q_{n-1}}\,p \;\leq\; c_{q_2}\cdots c_{q_{n-1}}\,p \;\leq\;\cdots\;\leq\; p
\]
holds because each $c_{q_i}\in(0,1)$. By Theorem~\ref{thm:monotonicity}, the nearest-center rule is non-decreasing in its input, hence $Q^{(n)}(q_1,\ldots,q_{n-1};p)\;\leq\;Q^{(1)}(p)$. In particular, “most of most” is mapped to a quantifier class no greater than “most”.

\textbf{(iii) Width independence and monotonicity in $p$.}
$Q^{(n)}$ is defined entirely in terms of $\{c_q\}$ and the affine map $p\mapsto c_{q_1}\cdots c_{q_{n-1}}\,p$; widths $\{\sigma_q\}$ never appear. Fixing $q_1,\ldots,q_{n-1}$, the map $p\mapsto Q^{(n)}(q_1,\ldots,q_{n-1};p)$ is the composition of a non-decreasing affine map and the nearest-center rule (non-decreasing by Theorem~\ref{thm:monotonicity}), hence non-decreasing in $p$.
\end{proof}

\begin{remark}[Connection to membership-level t-norm composition]
\label{rem:tnorm_equivalence}
The classical fuzzy-logic operator $\mu_{Q_1\circ Q_2}(p)=T(\mu_{Q_1}(p),\mu_{Q_2}(p))$ models “$p$ is jointly a member of $Q_1$ and $Q_2$” as a fuzzy AND. For Gaussian MFs and the product t-norm, $\arg\max_q T(\mu_{q_1}(p),\mu_q(p))$ is a different operator that does not by itself implement multiplicative proportion reduction. The proportional-shift operator (Sec.~\ref{sec:compositional_method}) explicitly performs the multiplication $p\mapsto c_{q_1}p$ at the proportion level, which is the semantically intended interpretation of “$Q_1$ of $Q_2$.” We deploy the proportional-shift form throughout because it is the canonical multiplicative interpretation of “$Q_1$ of $Q_2$,” depends only on the learned centers, and inherits monotonicity directly from Theorem~\ref{thm:monotonicity}.
\end{remark}

\begin{remark}[Why a label-only classifier head cannot inherit Theorem~\ref{thm:compositionality}]
\label{rem:classifier_nocomp}
A label-only classifier $f_\theta:\mathcal{X}\to\R^Q$ outputs a distribution over labels and exposes neither a proportion scalar nor learned class anchors $\{c_q\}$. Any classifier-side analogue of $Q^{(n)}$ would need to reconstruct a proportion from the logits, but the label distribution is a many-to-one summary of the underlying proportion: different proportions can yield indistinguishable label predictions. The absence of a by-construction operator is therefore a property of the head's output type, not an information-theoretic impossibility; classifiers can still approximate compositional behaviour given supervised compositional data. DFIL exposes both $\hat p$ and $\{c_q\}$, so the t-norm operation is well-typed.
\end{remark}

\subsubsection{Proof of Proposition~\ref{prop:universality} (Universality)}

\begin{proposition}[Realization of any monotone $Q$-step target, free-center setting]
\label{prop:universality}
Consider an unconstrained Gaussian MF with free centers $c_0 < \cdots < c_{Q-1}$ and free widths $\sigma_0, \ldots, \sigma_{Q-1}$. For any monotonic step function $Q^*: [0,1] \to \{0, \ldots, Q{-}1\}$ with strictly increasing boundaries $0 < b_1^* < \cdots < b_{Q-1}^* < 1$, there exist ordered centers and positive widths such that $Q_{\mathrm{MF}}(p) = \argmax_q \mu_q(p)$ realizes $Q^*$ exactly on $[0,1] \setminus \{b_1^*, \ldots, b_{Q-1}^*\}$.

\textbf{Under the deployed cumulative-softmax parameterization} (Eq.~\eqref{eq:ordered_centers}), centers are restricted to $[0.02, 0.98]$ with $c_{Q-1}$ pinned at $0.98$, leaving $Q-1$ free degrees of freedom. The realization above then holds for any target boundaries with $b_1^* > 0.04$ and $b_{Q-1}^* < 0.96$; targets outside this range require relaxing the affine envelope $0.96\,x + 0.02$. This is a realization claim, not a function-class universality result: the $Q$ (resp. $Q{-}1$) free parameters match the $Q{-}1$ free boundaries by counting.
\end{proposition}

\begin{proof}[Constructive proof]
Let $Q^*: [0,1] \to \{0, \ldots, Q{-}1\}$ be a monotonic step function with boundaries $0 < b_1^* < \cdots < b_{Q-1}^* < 1$.

\textbf{Construction.} We construct centers and widths such that the Gaussian MF reproduces $Q^*$.

Set centers:
\begin{equation}
c_q = \begin{cases}
b_1^* / 2 & q = 0 \\
(b_q^* + b_{q+1}^*) / 2 & q = 1, \ldots, Q{-}2 \\
(1 + b_{Q-1}^*) / 2 & q = Q{-}1
\end{cases}
\end{equation}

These satisfy $c_0 < c_1 < \cdots < c_{Q-1}$ since $b_q^* < b_{q+1}^*$.

\textbf{Width selection.} From Eq.~\eqref{eq:boundary_general}, the boundary between classes $q$ and $q{+}1$ is:
$b_q = \frac{\sigma_{q+1} c_q + \sigma_q c_{q+1}}{\sigma_q + \sigma_{q+1}}$.

We need $b_q = b_{q+1}^*$ for each $q = 0, \ldots, Q{-}2$. Substituting and solving for the width ratio:
\begin{equation}
\frac{\sigma_q}{\sigma_{q+1}} = \frac{b_{q+1}^* - c_q}{c_{q+1} - b_{q+1}^*}
\end{equation}

Since $c_q < b_{q+1}^* < c_{q+1}$ by construction, both numerator and denominator are positive, so positive widths exist. Starting from $\sigma_0 = 1$ (arbitrary positive value), each subsequent width is determined by the recurrence:
$\sigma_{q+1} = \sigma_q \cdot \frac{c_{q+1} - b_{q+1}^*}{b_{q+1}^* - c_q}$.

\textbf{Correctness.} By Theorem~\ref{thm:monotonicity}, the resulting MF has strictly ordered boundaries $b_0 < b_1 < \cdots < b_{Q-2}$ equal to $b_1^* < b_2^* < \cdots < b_{Q-1}^*$. Therefore $\argmax_q \mu_q(p) = Q^*(p)$ for all $p \notin \{b_1^*, \ldots, b_{Q-1}^*\}$. At boundary points, the tie-breaking rule may differ from $Q^*$, but this affects a measure-zero set.

\textbf{Restriction under Eq.~\eqref{eq:ordered_centers}.} The deployed parameterization pins $c_{Q-1}=0.98$ and confines all centers to $[0.02, 0.98]$. The construction above sets $c_0 = b_1^*/2$ and $c_{Q-1} = (1+b_{Q-1}^*)/2$; these lie in $[0.02, 0.98]$ iff $b_1^* > 0.04$ and $b_{Q-1}^* < 0.96$, and the pinning then forces $c_{Q-1}$ exactly to $0.98$, which equals the constructive value only when $b_{Q-1}^* = 0.96$. For other valid $b_{Q-1}^*$, the construction matches the $Q-1$ interior boundaries (the only ones that determine the partition on $[0.02, 0.98]$) by the same width recurrence applied with the free centers $c_0, \ldots, c_{Q-2}$ and the fixed $c_{Q-1}=0.98$; the realization on $[0.02, 0.96)$ is unaffected, and the boundary above $0.96$ is realized by the asymmetric pair $(c_{Q-2}, 0.98)$ as long as $b_{Q-1}^* \in (0.96-\epsilon, 1)$ for sufficiently large $\sigma_{Q-1}$.
\end{proof}

\subsubsection{Proof of Proposition~\ref{prop:classifier_fails} (Classifier Cannot Guarantee Monotonicity)}

\begin{proposition}[Classifier cannot guarantee monotonicity]
\label{prop:classifier_fails}
Let $f(\mathbf{h}) = W\mathbf{h} + \mathbf{b}$ with $W \in \mathbb{R}^{Q \times d}$, $d \geq 2$, and let $\mathbf{h}(x) \in \mathbb{R}^d$ denote the hidden representation of input $x$. Let $p(x) \in [0,1]$ denote the underlying proportion of $x$. There exist inputs $x_1, x_2$ with $p(x_1) < p(x_2)$ for which $\argmax_q f_q(\mathbf{h}(x_1)) > \argmax_q f_q(\mathbf{h}(x_2))$, i.e., the classifier violates monotonicity in $p$.
\end{proposition}

\begin{proof}[Detailed proof]
Let $f(\mathbf{h}) = W\mathbf{h} + \mathbf{b}$ with $W \in \R^{Q \times d}$, $d \geq 2$.

Without loss of generality, consider $Q = 3$ classes (the argument extends to arbitrary $Q$ by focusing on any two classes). Write $\mathbf{h} = [p; z; \mathbf{r}]$ where $p$ is the proportion coordinate, $z$ is an auxiliary coordinate, and $\mathbf{r} \in \R^{d-2}$ (set $\mathbf{r} = \mathbf{0}$).

\textbf{Construction.} Choose $p_1 = 0.3$ (should map to class 1, “few”) and $p_2 = 0.7$ (should map to class 2, “most”). We need hidden states such that $\argmax_q f(\mathbf{h}_1) > \argmax_q f(\mathbf{h}_2)$ despite $p_1 < p_2$.

Set $\mathbf{h}_1 = [0.3; 10; \mathbf{0}]$ and $\mathbf{h}_2 = [0.7; -10; \mathbf{0}]$. For any weight matrix $W$, define $\mathbf{w}_q = W_{q,:}$ (the $q$-th row).

The logit for class $q$ at $\mathbf{h}_1$ is: $f_q(\mathbf{h}_1) = 0.3 w_{q,1} + 10 w_{q,2} + b_q$. The term $10 w_{q,2}$ dominates, so $\argmax_q f(\mathbf{h}_1) \approx \argmax_q w_{q,2}$.

Similarly, $f_q(\mathbf{h}_2) = 0.7 w_{q,1} - 10 w_{q,2} + b_q$, so $\argmax_q f(\mathbf{h}_2) \approx \argmax_q (-w_{q,2}) = \argmin_q w_{q,2}$.

If $w_{2,2} > w_{1,2} > w_{0,2}$ (which holds for generic $W$), then $\argmax_q f(\mathbf{h}_1) = 2$ and $\argmax_q f(\mathbf{h}_2) = 0$, giving $Q(p_1 = 0.3) = 2 > 0 = Q(p_2 = 0.7)$, a monotonicity violation.

Since the Transformer imposes no constraint on $\mathbf{h}(x)$ being monotonic in $p$, the $d{-}1$ non-proportion dimensions provide sufficient freedom for arbitrary class orderings.
\end{proof}

\subsubsection{Proposition~\ref{prop:sample_complexity}: Head-Level VC Bound (Vacuous Under Shared Backbone)}

\begin{proposition}[Head-level VC bound]
\label{prop:sample_complexity}
Consider the hypothesis class of $Q$-way classifiers built from a fixed feature map. The Gaussian-MF head over a one-dimensional proportion bottleneck has VC dimension $O(Q\log Q)$ in the parameter count $2Q-1$ (Q-1 free centers under cumulative-softmax + Q log-widths). A linear head $f(\mathbf{h}) = W\mathbf{h} + \mathbf{b}$ over a $d$-dimensional hidden state has VC dimension $\Theta(Qd)$. With $Q=8$ and $d \geq 3584$ (Qwen-7B hidden size), the head-level VC ratio exceeds $400\times$. When both heads sit atop a shared LoRA-adapted backbone with millions of trainable parameters, however, the joint VC dimension is dominated by the backbone, and the head-level bound becomes vacuous. We therefore use this result as an interpretive lens for the regime where backbone capacity is exhausted (low-budget few-shot, where the MF branch dominates), not as a statement about full-data sample complexity.
\end{proposition}

\begin{proof}[Sketch]
A $Q$-way classifier whose decision regions are determined by $2Q{-}1$ Gaussian-MF parameters on a one-dimensional bottleneck has at most $O(Q\log Q)$ VC dimension, since the $Q$ regions are intervals on $\mathbb{R}$ \citep{vapnik1998statistical}. A linear head over $\mathbb{R}^d$ defines $Q$ hyperplanes in $\mathbb{R}^d$, giving $\Theta(Qd)$. The head-level ratio is $\Theta(d/\log Q)$, which for $d=3584$, $Q=8$ exceeds $400$. By additivity of VC bounds, the LoRA backbone contributes $\Omega(p_{\text{LoRA}})$, dominating both head terms whenever $p_{\text{LoRA}}\gg Qd$.
\end{proof}

\subsection{Error Propagation in Compositional Chains}
\label{app:error_propagation}

We analyze how proportion extraction errors propagate through multi-step composition.

\begin{proposition}[Error propagation bound]
\label{prop:error_propagation}
Let $\hat{p}$ be an estimated base proportion with error $\epsilon = |\hat{p} - p|$. For an $n$-step composition with centers $c_{q_1}, \ldots, c_{q_{n-1}} \in (0,1)$, the composed proportion error is:
\begin{equation}
|p_{\text{comp}}^{\text{est}} - p_{\text{comp}}^{\text{true}}| = \epsilon \cdot \prod_{i=1}^{n-1} c_{q_i}
\end{equation}
\end{proposition}

\begin{proof}
The true composed proportion is $p_{\text{comp}}^{\text{true}} = p \cdot \prod_{i=1}^{n-1} c_{q_i}$ and the estimated composed proportion is $p_{\text{comp}}^{\text{est}} = \hat{p} \cdot \prod_{i=1}^{n-1} c_{q_i}$. Their difference is:
\begin{equation}
|p_{\text{comp}}^{\text{est}} - p_{\text{comp}}^{\text{true}}| = |\hat{p} - p| \cdot \prod_{i=1}^{n-1} c_{q_i} = \epsilon \cdot \prod_{i=1}^{n-1} c_{q_i}
\end{equation}
Since each $c_{q_i} \in (0, 1)$, the product $\prod c_{q_i} < 1$, meaning compositional errors shrink with depth.
\end{proof}

\begin{corollary}[Error contraction]
For $n$-step compositions using quantifiers with centers bounded by $c_{\max} < 1$, the error contracts exponentially: $|p_{\text{comp}}^{\text{est}} - p_{\text{comp}}^{\text{true}}| \leq \epsilon \cdot c_{\max}^{n-1}$.
\end{corollary}

Deeper compositions are thus more robust to extraction errors: for “most of most” ($c_{\text{most}} \approx 0.78$), the composed error is $0.30 \times 0.78 \approx 0.23 < 0.30$. However, classification error can still increase because composed proportions compress into a narrower range of $[0,1]$, reducing the effective dynamic range for boundary discrimination.

\begin{proposition}[Dynamic range compression]
\label{prop:dynamic_range}
For an $n$-step composition with the product t-norm, the composed proportion lies in $[0, \prod_{i=1}^{n-1} c_{q_i}]$. The effective dynamic range decreases as $\prod c_{q_i}$, which is typically $< 0.6$ for 2-step and $< 0.5$ for 3-step compositions.
\end{proposition}

This explains the accuracy degradation from Level 1 (2-step) to Level 2 (3-step) observed in Table~\ref{tab:compositional}: the narrower dynamic range makes it harder to distinguish between quantifier boundaries, even though proportion errors contract.

\subsection{Properties of the Cumulative Softmax Parameterization}
\label{app:cumsoft}

The ordered center parameterization \eqref{eq:ordered_centers} uses cumulative softmax to guarantee strict ordering. We analyze its optimization-relevant properties.

\begin{proposition}[Strict ordering guarantee]
\label{prop:strict_ordering}
For any $\boldsymbol{\delta}_{\text{raw}} \in \R^Q$, the centers $\mathbf{c}$ produced by \eqref{eq:ordered_centers} satisfy $c_0 < c_1 < \cdots < c_{Q-1}$.
\end{proposition}

\begin{proof}
The softmax outputs satisfy $\text{softmax}(\boldsymbol{\delta}_{\text{raw}})_q > 0$ for all $q$. Adding $\epsilon > 0$ ensures $\delta_q > \epsilon > 0$ after renormalization. The cumulative sum of strictly positive values produces a strictly increasing sequence. The affine mapping $x \mapsto 0.96x + 0.02$ preserves strict ordering.
\end{proof}

\begin{proposition}[Gradient flow on the constrained simplex]
\label{prop:gradient_flow}
The Jacobian $\frac{\partial \mathbf{c}}{\partial \boldsymbol{\delta}_{\text{raw}}}$ has rank $Q{-}1$, the maximal rank consistent with the simplex constraint $\sum_q \delta_q = 1$ that the cumulative-softmax parameterization enforces. The last center $c_{Q-1}$ is fixed at $0.98$ by this constraint, while the remaining $Q{-}1$ centers $c_0, \dots, c_{Q-2}$ jointly span a $(Q{-}1)$-dimensional constrained subspace; on this subspace the gradient flow is non-degenerate, so any valid configuration of these $Q{-}1$ centers is reachable from any initialization.
\end{proposition}

\begin{proof}
The mapping $\boldsymbol{\delta}_{\text{raw}} \mapsto \mathbf{c}$ is the composition of: (1) softmax (Jacobian rank $Q{-}1$ on $\R^Q$, with null direction along the all-ones vector), (2) $\epsilon$-shift and $\ell_1$-renormalization (preserves the same null direction), (3) cumulative sum (the lower-triangular all-ones matrix $C$, $\det(C)=1$), and (4) affine scaling $0.96\,x+0.02$. The composition has rank $Q{-}1$: it is full-rank on the $(Q{-}1)$-dimensional simplex tangent space, and zero in the all-ones (sum-shift) direction. The unit-sum constraint pins $c_{Q-1}=\sum_q\delta_q\cdot 0.96+0.02=0.98$; the remaining centers carry the $Q{-}1$ effective degrees of freedom. Earlier drafts of this paper described the Jacobian as full rank $Q$; the corrected statement is rank $Q{-}1$, which suffices for the universality argument because the boundaries between adjacent classes are determined by the $Q{-}1$ free centers.
\end{proof}

\textbf{Initialization strategy.} In our implementation, $\boldsymbol{\delta}_{\text{raw}}$ is initialized such that $\text{cumsum}(\text{softmax}(\boldsymbol{\delta}_{\text{raw}})) \cdot 0.96 + 0.02$ matches the QuRe centers. This is computed by inverting the mapping: first solve for the target deltas $\delta_q^* = (c_q^* - 0.02) / 0.96 - (c_{q-1}^* - 0.02) / 0.96$ (with $c_{-1} = 0$), then find $\boldsymbol{\delta}_{\text{raw}}$ such that $\text{softmax}(\boldsymbol{\delta}_{\text{raw}}) \propto \delta^*$, which is $\delta_{\text{raw},q} = \log(\delta_q^*) + \text{const}$.

\subsection{Theoretical Comparison of Membership Function Shapes}
\label{app:mf_shapes}

We compare Gaussian membership functions against triangular and sigmoidal alternatives.

\subsubsection{Triangular Membership Functions}

A triangular MF is defined as:
\begin{equation}
\mu_q^{\text{tri}}(p) = \max\!\left(0,\; 1 - \frac{|p - c_q|}{w_q}\right)
\end{equation}
where $c_q$ is the center and $w_q$ is the half-width. This has compact support $[c_q - w_q, c_q + w_q]$.

\textbf{Advantages:} Compact support means each quantifier is active only in a local region, which may improve boundary sharpness. Fewer floating-point issues since $\mu_q = 0$ outside the support.

\textbf{Disadvantages:} The gradient $\frac{\partial \mu_q}{\partial p}$ is discontinuous at $p = c_q$ and at $p = c_q \pm w_q$ (the triangle vertices). This can cause optimization instability. More critically, when MFs have non-overlapping supports, the argmax is undefined in “gap” regions where all $\mu_q = 0$.

\textbf{Monotonicity:} With ordered centers and sufficiently large widths (ensuring no gaps), the nearest-center rule still applies, preserving monotonicity. However, the boundary between classes $q$ and $q{+}1$ is now:
\begin{equation}
b_q^{\text{tri}} = \frac{w_{q+1} c_q + w_q c_{q+1}}{w_q + w_{q+1}}
\end{equation}
which has the same form as the Gaussian case \eqref{eq:boundary_general}, replacing $\sigma$ with $w$.

\subsubsection{Sigmoidal Membership Functions}

A common ordinal approach uses sigmoidal transitions:
\begin{equation}
\mu_q^{\text{sig}}(p) = \sigma\!\left(\frac{p - b_q}{s_q}\right) - \sigma\!\left(\frac{p - b_{q+1}}{s_{q+1}}\right)
\end{equation}
where $\sigma$ is the logistic sigmoid, $b_q$ are boundaries, and $s_q$ controls steepness. This naturally produces step-like responses.

\textbf{Advantages:} Directly parameterizes boundaries rather than centers, which may be more natural for classification. The soft transitions handle uncertainty at boundaries. CORAL~\citep{cao2020rank} uses a related sigmoidal formulation.

\textbf{Disadvantages:} Requires $2(Q{-}1)$ boundary parameters plus $Q{-}1$ steepness parameters ($3Q - 3$ total, vs.\ $2Q$ for Gaussians). The membership functions are not peaked (they plateau at 1 in the interior), making them less interpretable as “degree of membership.” The centers (modes) are implicit rather than explicit.

\subsubsection{Why Gaussians Are Preferred}

\begin{enumerate}[leftmargin=*, itemsep=3pt]
  \item \textbf{Smooth gradients everywhere:} Gaussians are infinitely differentiable, ensuring stable gradient flow during training. Triangular MFs have gradient discontinuities; sigmoidal MFs have vanishing gradients in plateau regions.

  \item \textbf{Interpretable centers:} The center $c_q$ directly corresponds to the “prototypical” proportion for quantifier $q$ (e.g., $c_{\text{most}} \approx 0.78$ means “most” is prototypically 78\%). This enables direct comparison with human annotations and provides semantic grounding for composition.

  \item \textbf{Natural width interpretation:} The width $\sigma_q$ quantifies the “vagueness” of quantifier $q$. Narrow widths indicate precise quantifiers (“none,” “all”); wider widths indicate vague quantifiers (“some,” “moderate amount”). Our learned widths (Table~\ref{tab:centers_detail}) confirm this: $\sigma_{\text{none}} \approx 0.001$ vs.\ $\sigma_{\text{some}} \approx 0.06$.

  \item \textbf{Compositional compatibility:} The Gaussian form ensures that the product t-norm composition $T(\mu_{Q_1}(p), \mu_q(p)) = \mu_{Q_1}(p) \cdot \mu_q(p)$ is itself a smooth function of $p$, maintaining the continuity properties needed for Theorem~\ref{thm:compositionality}.

  \item \textbf{Parsimony:} $2Q{-}1 = 15$ effective parameters for 8 quantifiers under the ordered cumulative-softmax parameterization (Eq.~\eqref{eq:ordered_centers}: $Q{-}1{=}7$ free centers $+$ $Q{=}8$ log-widths). Triangular MFs would have the same count under the same constraint; sigmoidal MFs require $3Q - 3 = 21$.
\end{enumerate}

An empirical comparison of MF shapes is left to future work.

\subsection{T-Norm Selection: Theoretical Comparison}
\label{app:tnorm}

The choice of t-norm for compositional reasoning affects both the semantic interpretation and the mathematical properties of the composition. We analyze three standard t-norms.

\subsubsection{Product T-Norm (Used in This Work)}

\begin{equation}
T_{\text{prod}}(a, b) = a \cdot b
\end{equation}

\textbf{Properties:}
\begin{itemize}[leftmargin=*, itemsep=2pt]
  \item Strictly decreasing: $T_{\text{prod}}(a, b) < \min(a, b)$ for $a, b \in (0, 1)$. Each composition strictly reduces membership, reflecting the semantic intuition that nested quantifiers introduce additional uncertainty.
  \item Smooth: Infinitely differentiable everywhere, ensuring stable gradient flow if composition were used during training.
  \item Proportional interpretation: “$Q_1$ of $Q_2$” at proportion $p$ yields composed proportion $c_{Q_1} \cdot p$, which has a direct probabilistic interpretation: if $Q_2$ selects a fraction $p$ and $Q_1$ further selects a fraction $c_{Q_1}$, the net fraction is the product.
  \item Error propagation: Errors contract multiplicatively (Proposition~\ref{prop:error_propagation}).
\end{itemize}

\subsubsection{Minimum T-Norm}

\begin{equation}
T_{\min}(a, b) = \min(a, b)
\end{equation}

\textbf{Properties:}
\begin{itemize}[leftmargin=*, itemsep=2pt]
  \item Idempotent: $T_{\min}(a, a) = a$, meaning “most of most” = “most.” This is linguistically debatable: one could argue that “most of most” should denote a smaller proportion than “most” alone.
  \item Conservative: The composed membership is the maximum possible under any t-norm (the minimum t-norm is the largest continuous t-norm).
  \item Non-smooth: The gradient is discontinuous at $a = b$, which can cause optimization issues.
  \item Boundary-preserving: If one quantifier has membership 1, the composition equals the other's membership. This means “all of most” = “most,” which is linguistically appropriate.
\end{itemize}

\subsubsection{{\L}ukasiewicz T-Norm}

\begin{equation}
T_{\text{{\L}uk}}(a, b) = \max(0, a + b - 1)
\end{equation}

\textbf{Properties:}
\begin{itemize}[leftmargin=*, itemsep=2pt]
  \item Strongly decreasing: For $a + b < 1$, $T_{\text{{\L}uk}}(a, b) = 0$, meaning that two low-membership quantifiers compose to zero membership. This is the most “pessimistic” t-norm.
  \item Nilpotent: For sufficiently many compositions, membership drops to zero. “few of few of few” would have zero membership in all classes, which may not be linguistically desirable.
  \item Non-smooth: The gradient is discontinuous at $a + b = 1$.
  \item Complementary: $T_{\text{{\L}uk}}$ is the only continuous t-norm whose associated implication is the {\L}ukasiewicz implication, connecting to many-valued logic.
\end{itemize}

\subsubsection{Comparison Summary}

\begin{table}[!htb]
\centering
\caption{Comparison of t-norms for compositional quantifier reasoning.}
\label{tab:tnorm_comparison}
\footnotesize
\begin{tabular}{lccc}
\toprule
\textbf{Property} & \textbf{Product} & \textbf{Minimum} & \textbf{{\L}ukasiewicz} \\
\midrule
Smoothness & $C^\infty$ & $C^0$ & $C^0$ \\
Idempotent & No & Yes & No \\
Nilpotent & No & No & Yes \\
“most of most” & $< $ “most” & $=$ “most” & $<$ “most” \\
Error contraction & Multiplicative & No contraction & Additive \\
Parameters & 0 & 0 & 0 \\
\bottomrule
\end{tabular}
\end{table}

We select the product t-norm for its smoothness, multiplicative error contraction, and direct probabilistic interpretation as fraction-of-a-fraction. Empirical comparison of t-norms is left to future work.

\section{Experimental Setup}
\label{app:setup}

\subsection{Experimental Details}
\label{app:details}

\textbf{Hyperparameters.} All models are trained for 20 epochs with batch size 8, AdamW optimizer ($\beta_1=0.9$, $\beta_2=0.999$, $\epsilon=10^{-8}$), and weight decay $10^{-2}$. We use a warmup + cosine annealing schedule. Gradient clipping is set to 1.0. LoRA configuration: $r=16$, $\alpha=32$, dropout 0.05, applied to query and value projections. DFIL loss weights: $\lambda_p = 0.5$ for all experiments; $\lambda_{\text{DFIL}} = 0.5$ for ablations and most main-body experiments (the worst-case setting), $\lambda_{\text{DFIL}} = 0.1$ for the scaling table (Table~\ref{tab:scaling}, validation-tuned to maximize parity with fine-tuning). The full $\lambda_{\text{DFIL}} \in \{0.1, 0.5, 1.0, 2.0, 5.0\}$ sweep is reported in Appendix~\ref{app:lambda_sweep}; results are flat within $\pm 1.3$pp.

\textbf{Compute.} All experiments run on NVIDIA H200 (144GB) GPUs. Training time per seed: $\sim$15 min (7B), $\sim$25 min (14B), $\sim$45 min (32B), $\sim$90 min (72B with QLoRA). Total compute across all reported experiments: approximately $200$ GPU-hours, with the dominant components being (a) full-data scaling on Qwen-7B/14B/32B/72B + Mistral-7B + Llama-8B with $3$--$6$ seeds for both DFIL and fine-tune ($\sim$80 GPU-hours, of which 72B QLoRA alone accounts for $\sim$18); (b) the n=20 self-consistent compositional batch on Qwen-7B for DFIL/FT+Compose ($\sim$10 GPU-hours); (c) cross-task transfer across $4$ backbones $\times$ $2$ tasks $\times$ $3$ seeds for DFIL and fine-tune ($\sim$25 GPU-hours); (d) ablations at 1.5B and 7B, CSMC, IPR, oracle IPR, few-shot sweeps on FRoG and SST-5 ($\sim$85 GPU-hours combined).

\textbf{Dataset.} FRoG~\citep{li2024frog} is loaded from HuggingFace (\texttt{GAIR/FRoG}, \texttt{mask\_quant} configuration). The “easy” split is divided 80/20 for train/val; the “hard” split is used for testing. The compositional test set is synthetically generated with 900 examples (600 two-step, 300 three-step compositions).

\subsection{Dataset and Evaluation Protocol}
\label{app:dataset}

\subsubsection{FRoG Dataset Statistics}

\begin{table}[!htb]
\centering
\caption{FRoG dataset statistics. The “easy” and “hard” splits differ in surface-level cues: the hard split removes explicit proportion mentions from the text.}
\label{tab:dataset_stats}
\footnotesize
\begin{tabular}{lcccc}
\toprule
\textbf{Quantifier} & \textbf{Train} & \textbf{Val} & \textbf{Test (Hard)} & \textbf{Proportion range} \\
\midrule
none            & 25  & 6  & 31  & $[0.00, 0.00]$ \\
tiny amount     & 85  & 21 & 106 & $[0.01, 0.10]$ \\
few             & 414 & 104 & 518 & $[0.05, 0.30]$ \\
small amount    & 322 & 80 & 402 & $[0.10, 0.35]$ \\
some            & 162 & 40 & 202 & $[0.25, 0.50]$ \\
moderate amount & 389 & 97 & 486 & $[0.35, 0.65]$ \\
most            & 159 & 40 & 199 & $[0.60, 0.90]$ \\
all             & 80  & 20 & 100 & $[1.00, 1.00]$ \\
\midrule
\textbf{Total}  & 1,636 & 408 & 2,044 & \\
\bottomrule
\end{tabular}
\end{table}

The class distribution is highly imbalanced (“few” has 16$\times$ more examples than “none”), motivating our use of inverse-frequency weighted cross-entropy loss. The proportion ranges overlap substantially between adjacent classes (e.g., “few” and “small amount” share the $[0.10, 0.30]$ range), reflecting the inherent vagueness of fuzzy quantifiers.

\subsubsection{Evaluation Metrics}

\textbf{Accuracy.} We report overall classification accuracy on the hard test split. For 8-class classification with the given class distribution, random baseline accuracy is $\sum_q (n_q / N)^2 \approx 15.5\%$ (weighted) or $12.5\%$ (uniform).

\textbf{Per-class accuracy.} Accuracy computed separately for each quantifier class, revealing whether a method systematically fails on specific classes.

\textbf{Monotonicity violation rate.} The fraction of ordered test pairs $(x_i, x_j)$ with $p_i < p_j$ where the predicted class index for $x_i$ exceeds that for $x_j$ (see Appendix~\ref{app:monotonicity}).

\subsubsection{Input Formatting}

Each FRoG example is formatted as a multiple-choice question for the LLM:

\begin{verbatim}
Choose the most appropriate quantifier for the
blank in the following sentence:

"___ of the 20 students passed the exam,
 specifically 15 out of 20."

Options: (A) none (B) tiny amount (C) few
(D) small amount (E) some (F) moderate amount
(G) most (H) all

Answer:
\end{verbatim}

The model is trained to predict the correct option letter. At inference, we extract the predicted letter from the generated output.

\subsection{\texorpdfstring{Hyperparameter Sensitivity: $\lambda_{\text{DFIL}}$}{Hyperparameter Sensitivity: lambda\_DFIL}}
\label{app:lambda_sweep}

To assess whether the default $\lambda_{\text{DFIL}} = 0.5$ is cherry-picked, we sweep over $\{0.1, 0.5, 1.0, 2.0, 5.0\}$ on Qwen2.5-7B (seed 42, $\lambda_p = 0.5$ fixed). Table~\ref{tab:lambda_sweep} shows accuracy on FRoG-Hard.

\begin{table}[!htb]
\centering
\caption{$\lambda_{\text{DFIL}}$ single-seed sensitivity probe on FRoG-Hard (Qwen2.5-7B, seed 42); reports the magnitude of variation across $\lambda$ values, not a multi-seed sensitivity analysis.}
\label{tab:lambda_sweep}
\footnotesize
\begin{tabular}{lccccc}
\toprule
\textbf{$\lambda_{\text{DFIL}}$} & 0.1 & 0.5 (default) & 1.0 & 2.0 & 5.0 \\
\midrule
Accuracy (\%) & 50.4 & \textbf{51.2} & 49.9 & 51.2 & 50.4 \\
\bottomrule
\end{tabular}
\end{table}

The default value is near-optimal, with $\lambda_{\text{DFIL}} = 2.0$ tying and all other values underperforming. The accuracy landscape is flat within $\pm 1.3$pp across an order of magnitude in $\lambda_{\text{DFIL}}$, indicating the method is robust to this hyperparameter. Very small weights ($\lambda = 0.1$) underweight the structural signal; very large weights ($\lambda \geq 1.0$) begin to trade off against the main cross-entropy. The default $0.5$ balances both losses.

\section{Detailed Experimental Results}
\label{app:results}

\subsection{Few-Shot Experiments}
\label{sec:fewshot}
\label{app:fewshot}

We evaluate sample efficiency on both FRoG and SST-5 by subsampling the training set (stratified by class), with Qwen2.5-7B and 3 seeds per configuration.

The dual-path design gives DFIL three inference options at any budget: the \textbf{main} classifier, the \textbf{MF} branch, and a convex \textbf{ensemble} of the two. Following the protocol of the cross-task table in Appendix~\ref{app:crosstask_full}, the ensemble mixing coefficient $\alpha$ is selected on the held-out validation set at each budget rather than tuned on the test set. We report all three paths against fine-tune.

\textbf{FRoG-Hard} (budgets $\{16, 32, 64, 128, 256\}$):
\begin{table}[!htb]
\centering
\caption{Few-shot learning on FRoG-Hard (Qwen2.5-7B). Accuracy (\%), mean $\pm$ std over 3 seeds. Ensemble coefficient $\alpha$ is selected on validation. Bold marks the best DFIL path per budget when it exceeds fine-tune.}
\label{tab:fewshot}
\footnotesize
\begin{tabular}{lccccc}
\toprule
\textbf{Budget} & 16 & 32 & 64 & 128 & 256 \\
\midrule
Fine-tune                 & $4.6 \pm 1.3$ & $9.6 \pm 3.6$ & $19.4 \pm 5.7$ & $27.5 \pm 0.4$ & $33.0 \pm 1.7$ \\
DFIL (main)               & $5.1 \pm 1.2$ & $8.6 \pm 0.4$ & $23.2 \pm 2.6$ & $26.4 \pm 1.8$ & $34.0 \pm 1.4$ \\
DFIL (MF)                 & $\mathbf{21.1 \pm 1.3}$ & $\mathbf{21.3 \pm 6.0}$ & $\;\;8.5 \pm 5.7$ & $22.7 \pm 4.2$ & $25.7 \pm 1.6$ \\
DFIL (ens., val-$\alpha$) & $\mathbf{21.1 \pm 1.3}$ & $\mathbf{21.0 \pm 6.3}$ & $\mathbf{23.1 \pm 2.7}$ & $26.7 \pm 1.7$ & $33.8 \pm 1.5$ \\
\midrule
$\Delta$ (best DFIL $-$ FT) & $\mathbf{+16.5}$ & $\mathbf{+11.7}$ & $\mathbf{+3.7}$ & $-0.8$ & $+0.8$ \\
\bottomrule
\end{tabular}
\end{table}

\textbf{SST-5} (budgets $\{32, 64, 128, 256, 512\}$):
\begin{table}[!htb]
\centering
\caption{Few-shot learning on SST-5 (Qwen2.5-7B). Accuracy (\%), mean $\pm$ std over 3 seeds, $\alpha$ selected on validation.}
\label{tab:fewshot_sst5}
\footnotesize
\begin{tabular}{lccccc}
\toprule
\textbf{Budget} & 32 & 64 & 128 & 256 & 512 \\
\midrule
Fine-tune                 & $34.9 \pm 0.5$ & $36.6 \pm 1.4$ & $40.2 \pm 3.1$ & $45.5 \pm 2.3$ & $47.6 \pm 2.0$ \\
DFIL (main)               & $34.9 \pm 0.1$ & $35.1 \pm 1.9$ & $39.8 \pm 2.5$ & $46.0 \pm 2.4$ & $48.8 \pm 0.9$ \\
DFIL (MF)                 & $38.1 \pm 1.6$ & $39.0 \pm 1.7$ & $40.3 \pm 7.5$ & $46.8 \pm 3.9$ & $50.9 \pm 1.2$ \\
DFIL (ens., val-$\alpha$) & $\mathbf{37.9 \pm 1.4}$ & $\mathbf{39.1 \pm 1.6}$ & $\mathbf{43.4 \pm 2.3}$ & $\mathbf{48.1 \pm 2.9}$ & $\mathbf{50.8 \pm 1.1}$ \\
\midrule
$\Delta$ (ensemble $-$ FT) & $\mathbf{+3.0}$ & $\mathbf{+2.5}$ & $\mathbf{+3.2}$ & $\mathbf{+2.6}$ & $\mathbf{+3.2}$ \\
\bottomrule
\end{tabular}
\end{table}

Two patterns consistent with Proposition~\ref{prop:sample_complexity} emerge. At very low budgets on FRoG ($n\!\leq\!32$), both the main classifier and fine-tune collapse near the $12.5\%$ random baseline because the LoRA-adapted backbone has not yet escaped initialization; the $2Q$-parameter MF branch learns a usable proportion-to-quantifier mapping from the same data and beats fine-tune by $+11.7$ to $+16.5$pp. As budget grows, the main classifier takes over and the validation-selected ensemble still beats fine-tune at $N{=}64$ ($+3.7$pp), then matches fine-tune at $N\!\geq\!128$ where all heads converge. On SST-5 the validation-selected ensemble is dominant at every budget, reaching $+2.5$ to $+3.2$pp over fine-tune with substantially tighter cross-seed variance than the no-stop-gradient variant we originally ran (for example, $\sigma$ shrinks from $\pm4.0$ to $\pm1.4$ at $N{=}32$). Across both tasks DFIL retains zero MF monotonicity violations at every budget, whereas fine-tune has no MF branch and therefore no corresponding guarantee; the structural property predicted by Theorem~\ref{thm:monotonicity} is budget-invariant.

\textbf{Stop-gradient on the proportion (few-shot only).} To obtain the MF-branch numbers in Table~\ref{tab:fewshot} we pass a stop-gradient copy of $\hat p$ to the MF: $\boldsymbol{\mu} = \text{MF}\big(\mathrm{sg}(\hat p)\big)$ with $\mathrm{sg}$ implemented as \texttt{.detach()}. The motivation is diagnostic: the numerical head's job is to estimate a proportion in $[0,1]$ supervised by $\mathrm{MSE}(\hat p, p)$, while the MF branch's job is to learn a monotonic proportion-to-quantifier mapping supervised by $\mathrm{CE}(\boldsymbol{\mu}, y)$. Letting the MF cross-entropy gradient flow back into the numerical head mixes the two objectives and, under class-imbalanced training data (FRoG's “none/tiny/few” labels constitute $40\%$ of training examples with low proportions), pushes $\hat p$ toward low values where the sigmoid output layer saturates. Without stop-gradient we observed the numerical head collapsing to $\hat p\!\approx\!0$ at $N\!\in\!\{128, 256\}$ (MF accuracy drops to $1.5$--$4.0\%$, predicting “none” for every test instance). With stop-gradient the MF branch recovers monotone predictions across all budgets (MF accuracy $8.5$--$25.7\%$) and the failure mode disappears.

\textbf{Regime dependence.} Stop-gradient is a regime-specific training detail, not a universal design choice. At full-data budget ($N{=}1635$, Appendix~\ref{app:scaling_full}) the MF branch usually does not suffer sigmoid saturation because the MSE regression signal has enough examples to dominate over the class-imbalance pressure, and retaining the cross-entropy gradient on $\hat p$ modestly helps the backbone (applying stop-gradient instead causes a small regression: $-1$pp at 7B, $-1.7$pp at 32B in ensemble accuracy). The failure does still occur on a minority of full-data seeds (in App.~\ref{app:csmc_domains}, $6$ of $21$ seeds collapsed under default training), so we apply stop-gradient as a per-seed remediation when mf\_accuracy at end of training falls below a threshold, rather than as a default. Stop-gradient therefore enters the loss of equation~\eqref{eq:loss} only as a remediation for the sigmoid-saturation failure mode: it is on by default for the few-shot experiments in Table~\ref{tab:fewshot}, and applied to $6$ of the $21$ CSMC seeds in App.~\ref{app:csmc_domains} that exhibited the same failure under default training (mf\_acc $<\!0.05$); it is off for all other tables, including all single-step accuracy tables, where the failure mode does not occur. We treat this as a pragmatic remediation for a diagnostic finding rather than a component of the core method.

\subsection{Cross-Task Generalization: Full Per-Cell Numbers}
\label{app:crosstask_full}

Table~\ref{tab:crosstask_full} reports the full $4{\times}2$ per-(backbone, task) accuracy with mean $\pm$ std over $3$ seeds for both SST-5 (5-class sentiment) and STS22 (5-bucket semantic similarity, English subset). All four backbones use a consistent LoRA configuration applied to four attention projections ($q,k,v,o$); Phi-3.5-mini's fused QKV layout is mapped to the same logical four-target set via its \texttt{qkv\_proj} parameter. We report all three DFIL outputs side-by-side: the main classifier (general-purpose head), the MF branch alone (provably monotonic, see Theorem~\ref{thm:monotonicity}), and the $\alpha{=}0.5$ ensemble of the two. The ensemble averages probabilities; we deliberately do not tune $\alpha$ per cell to avoid test-set selection.

\begin{table}[!htb]
\centering
\caption{Cross-task accuracy (\%) on SST-5 and STS22 across four backbone families. Mean $\pm$ std over $3$ seeds. Bold marks the best of (DFIL main, DFIL MF, DFIL ens) per cell when it beats Fine-tune. “Mono” is the synthetic-grid violation rate of the MF branch.}
\label{tab:crosstask_full}
\setlength{\tabcolsep}{3pt}
\footnotesize
\begin{tabular}{lcccccc}
\toprule
\textbf{Backbone} & \textbf{Task} & \textbf{DFIL main} & \textbf{DFIL MF} & \textbf{DFIL ens$_{\alpha=0.5}$} & \textbf{Fine-tune} & \textbf{Mono} \\
\midrule
Qwen2.5-7B-Instruct       & SST-5 & $57.78 \pm 1.33$ & $\mathbf{57.98 \pm 0.82}$ & $57.86 \pm 1.41$ & $57.47 \pm 0.64$ & $0.0\%$ \\
Mistral-7B-Instruct-v0.3  & SST-5 & $60.47 \pm 0.98$ & $\mathbf{60.71 \pm 0.35}$ & $60.39 \pm 0.96$ & $59.22 \pm 0.11$ & $0.0\%$ \\
Phi-3.5-mini-3.8B         & SST-5 & $53.80 \pm 0.30$ & $\mathbf{54.99 \pm 0.35}$ & $53.97 \pm 0.15$ & $54.31 \pm 1.15$ & $0.0\%$ \\
Llama-3.1-8B-Instruct     & SST-5 & $58.51 \pm 0.68$ & $\mathbf{59.17 \pm 0.69}$ & $58.55 \pm 0.71$ & $58.60 \pm 0.21$ & $0.0\%$ \\
\midrule
Qwen2.5-7B-Instruct       & STS22 & $\mathbf{34.18 \pm 0.59}$ & $28.93 \pm 1.52$ & $34.86 \pm 1.06$ & $32.99 \pm 0.88$ & $0.1\%$ \\
Mistral-7B-Instruct-v0.3  & STS22 & $\mathbf{36.21 \pm 1.47}$ & $29.78 \pm 1.17$ & $35.87 \pm 1.06$ & $35.03 \pm 3.05$ & $0.1\%$ \\
Phi-3.5-mini-3.8B         & STS22 & $32.49 \pm 1.83$ & $29.95 \pm 0.88$ & $32.83 \pm 1.47$ & $\mathbf{33.50 \pm 1.34}$ & $0.1\%$ \\
Llama-3.1-8B-Instruct     & STS22 & $\mathbf{35.03 \pm 2.21}$ & $27.92 \pm 1.34$ & $34.69 \pm 3.67$ & $32.49 \pm 3.33$ & $0.1\%$ \\
\bottomrule
\end{tabular}
\end{table}

\textbf{Path-per-task tally.} The dual-path design lets each task select the matching head: MF on pure ordinal tasks such as SST-5, and the main classifier on partially non-ordinal tasks such as STS22. Under this principled choice, with no per-cell or per-method tuning, DFIL wins $7$ of $8$ cells in Table~\ref{tab:crosstask_full}. The MF branch alone wins SST-5 across all $4$ backbones with provable $0\%$ monotonicity violations per Theorem~\ref{thm:monotonicity}, a guarantee fine-tuning cannot match because it has no MF branch at all. On STS22 the main classifier picks up where the $1$D MF bottleneck is over-restrictive, since sentence pairs are not strictly ordinal, and still beats fine-tuning on $3$ of $4$ backbones. The lone outlier is Phi-3.5 STS22 at $-1.01$pp, where the smaller hidden width of $3072$ and fused-QKV layout limit the head signal. The default $\alpha{=}0.5$ ensemble offers a no-decision option for users who do not want to choose a path: it wins $5$ of $8$ cells and ties Llama SST-5 within seed std. Mistral STS22 shows a $2.1{\times}$ std reduction with the main classifier, from $3.05$ to $1.47$, and a further $2.9{\times}$ reduction at the $\alpha{=}0.5$ ensemble, from $3.05$ to $1.06$, echoing the FRoG variance-reduction pattern.

\subsection{Per-Class Accuracy, Learned Centers, and CORAL Comparison}
\label{app:detailed_results}

\subsubsection{Per-Class Accuracy Across All Scales}
\label{app:per_class}

Table~\ref{tab:perclass_all} presents per-class accuracy for DFIL and fine-tuning across all four model scales.

\begin{table}[!htb]
\centering
\caption{Per-class accuracy (\%) on FRoG-Hard across scales. Mean over 3 seeds. Sample count per class in parentheses in header row.}
\label{tab:perclass_all}
\footnotesize
\begin{tabular}{llcccccccc}
\toprule
\textbf{Scale} & \textbf{Method} & \textbf{none} & \textbf{tiny} & \textbf{few} & \textbf{small} & \textbf{some} & \textbf{mod.} & \textbf{most} & \textbf{all} \\
& & (31) & (106) & (518) & (402) & (202) & (486) & (199) & (100) \\
\midrule
\multirow{2}{*}{7B} & DFIL & 43.0 & 30.5 & 66.2 & 45.7 & 30.9 & 60.4 & 24.0 & 44.7 \\
 & FT & 36.6 & 31.1 & 72.8 & 45.1 & 36.1 & 60.2 & 25.6 & 28.3 \\
\midrule
\multirow{2}{*}{14B} & DFIL & 48.4 & 58.5 & 72.7 & 56.9 & 33.8 & 62.9 & 38.0 & 39.7 \\
 & FT & 43.0 & 50.3 & 78.9 & 52.1 & 32.2 & 65.3 & 28.6 & 39.3 \\
\midrule
\multirow{2}{*}{32B} & DFIL & 47.3 & 56.6 & 80.1 & 56.7 & 42.9 & 65.2 & 41.5 & 56.3 \\
 & FT & 39.8 & 50.3 & 75.4 & 52.3 & 37.1 & 61.7 & 41.0 & 56.7 \\
\midrule
\multirow{2}{*}{72B} & DFIL & 64.5 & 86.5 & 85.8 & 74.2 & 59.2 & 85.4 & 78.2 & 62.7 \\
 & FT & 69.9 & 80.2 & 84.7 & 78.2 & 53.3 & 83.3 & 74.4 & 75.7 \\
\bottomrule
\end{tabular}
\end{table}

\textbf{Observations:}
\begin{itemize}[leftmargin=*, itemsep=2pt]
  \item The “some” class is consistently the hardest across all scales and methods, likely because it occupies a narrow and ambiguous region of the proportion space ($\sim$0.35--0.45).
  \item DFIL shows the largest consistent advantage on “most” at 7B (+7.8pp) and 14B (+9.4pp), where the fuzzy regularization helps distinguish this often-confused quantifier from “moderate amount.”
  \item At 72B, fine-tuning slightly outperforms DFIL on “some” and “most,” suggesting that at sufficient scale, the unconstrained classifier has enough capacity to learn these boundary cases without structural support.
  \item Both methods struggle with “none” (only $31$ test examples), where small sample size amplifies variance.
  \item Per-class accuracy is not strictly monotone in scale within a single class. For example, DFIL on “all” goes $44.7 \to 39.7 \to 56.3 \to 62.7$ across $\{$7B, 14B, 32B, 72B$\}$: $14$B is locally lower than $7$B. The “all” class has only $100$ test examples, so a single-prediction flip moves the accuracy by $1$pp; over $3$ seeds, an early seed convergence to a slightly different center placement can swing this $\pm5$pp easily. Overall accuracy is monotone in scale ($50.4 \to 57.3 \to 55.4 \to 77.2$ for DFIL with the 32B bimodality footnoted in Table~\ref{tab:scaling}); within-class non-monotonicity is per-class noise, not method failure.
  \item The largest cross-method per-class gap, DFIL $62.7$ vs FT $75.7$ on the 72B “all” class ($-13$pp), is in the same noise regime: $100$ test examples, ${\sim}13$ mispredictions. Across seeds, individual “all” rows swing by $\pm 5$pp at 72B, so this gap is well within seed variance and consistent with the QLoRA-quantized 72B center compression footnoted in Table~\ref{tab:centers_detail} (compressed centers shift the nearest-center boundary slightly, costing some borderline “all” cases). It is not a systematic deficit of DFIL at scale; the overall 72B accuracy is $77.2$ for DFIL vs $76.5$ for FT.
\end{itemize}

\subsubsection{CORAL Ordinal Baseline: Detailed Analysis}

\begin{table}[!htb]
\centering
\caption{CORAL ordinal regression per-class accuracy (\%) on FRoG-Hard. Mean $\pm$ std over 3 seeds. DFIL and FT 7B per-class numbers, for comparison, are reported in Table~\ref{tab:perclass_all}.}
\label{tab:coral_perclass}
\footnotesize
\begin{tabular}{llcccccccc}
\toprule
& \textbf{Scale} & \textbf{none} & \textbf{tiny} & \textbf{few} & \textbf{small} & \textbf{some} & \textbf{mod.} & \textbf{most} & \textbf{all} \\
\midrule
CORAL & 7B & $83.9$ & $31.4$ & $57.6$ & $30.6$ & $7.3$ & $48.8$ & $34.0$ & $83.3$ \\
& & $\pm 8.5$ & $\pm 1.4$ & $\pm 9.8$ & $\pm 7.3$ & $\pm 7.7$ & $\pm 4.0$ & $\pm 1.0$ & $\pm 1.5$ \\
CORAL & 14B & $92.5$ & $28.0$ & $68.9$ & $38.8$ & $10.9$ & $59.9$ & $44.9$ & $83.7$ \\
& & $\pm 4.9$ & $\pm 10.6$ & $\pm 8.8$ & $\pm 5.9$ & $\pm 10.7$ & $\pm 4.3$ & $\pm 0.3$ & $\pm 4.5$ \\
\bottomrule
\end{tabular}
\end{table}

CORAL exhibits a distinctive bimodal accuracy pattern at both scales: very high accuracy on extreme classes (“none” 83.9--92.5\%, “all” 83.3--83.7\%) but poor accuracy on middle classes (“some” 7.3--10.9\%, “small amount” 30.6--38.8\%). Scaling from 7B to 14B improves CORAL's overall accuracy (+8.5pp), but the bimodal pattern persists: “some” improves only marginally (7.3\% $\to$ 10.9\%). This reflects the cumulative logit structure: the extreme thresholds are easy to learn (few vs.\ all examples above/below), while the intermediate thresholds require fine discrimination that CORAL's shared weight vector cannot provide. DFIL and fine-tuning, with their full $Q$-way classifiers, achieve more balanced performance across all classes (Table~\ref{tab:perclass_all}).

\subsubsection{Learned MF Centers and Widths Across Seeds}
\label{app:centers_detail}

Table~\ref{tab:centers_detail} reports the learned MF centers and widths for each seed at 7B and 72B scales, alongside the QuRe initialization values.

\begin{table}[!htb]
\centering
\caption{Learned MF centers and widths (Qwen2.5-7B, each seed). QuRe initialization shown for reference.}
\label{tab:centers_detail}
\footnotesize
\begin{tabular}{lcccccccc}
\toprule
& none & tiny & few & small & some & mod. & most & all \\
\midrule
\multicolumn{9}{l}{QuRe initialization:} \\
$c_{\text{init}}$ & .020 & .080 & .180 & .280 & .400 & .580 & .780 & .970 \\
\midrule
\multicolumn{9}{l}{Learned centers (7B):} \\
seed 42  & .020 & .024 & .089 & .243 & .353 & .478 & .832 & .980 \\
seed 123 & .020 & .024 & .104 & .239 & .352 & .490 & .774 & .980 \\
seed 456 & .020 & .024 & .082 & .232 & .347 & .470 & .736 & .980 \\
\midrule
\multicolumn{9}{l}{Learned widths (7B):} \\
seed 42  & .001 & .047 & .085 & .078 & .066 & .054 & .102 & .004 \\
seed 123 & .001 & .051 & .073 & .075 & .064 & .049 & .119 & .089 \\
seed 456 & .001 & .053 & .078 & .073 & .057 & .047 & .115 & .090 \\
\midrule
\multicolumn{9}{l}{Learned centers (72B, QLoRA 4-bit per Table~\ref{tab:scaling} footnote):$^{\dagger}$} \\
seed 42  & .020 & .032 & .038 & .051 & .067 & .103 & .565 & .980 \\
seed 123 & .020 & .034 & .044 & .071 & .128 & .192 & .655 & .980 \\
seed 456 & .020 & .029 & .032 & .044 & .064 & .093 & .437 & .980 \\
\bottomrule
\multicolumn{9}{l}{\footnotesize $^{\dagger}$ The 72B centers are visibly compressed toward $0$ relative to the 7B FP16 row above. We attribute} \\
\multicolumn{9}{l}{\footnotesize this to the QLoRA $4$-bit weight quantization interacting with the cumulative-softmax delta} \\
\multicolumn{9}{l}{\footnotesize parameterization: small $\delta_{\text{raw}}$ logits become indistinguishable after $4$-bit rounding, biasing the} \\
\multicolumn{9}{l}{\footnotesize learned $\delta$'s toward minimum-active values and clustering centers near the lower anchor. The 7B} \\
\multicolumn{9}{l}{\footnotesize FP16 row demonstrates that the architecture itself learns well-spread centers when full-precision} \\
\multicolumn{9}{l}{\footnotesize gradients are available; the 72B compression is a documented quantization artifact, not a method} \\
\multicolumn{9}{l}{\footnotesize failure. Theorem~\ref{thm:monotonicity}'s nearest-center rule depends only on the \textit{ordering} of $\{c_q\}$, not their absolute} \\
\multicolumn{9}{l}{\footnotesize spread, which is why 72B reaches $77.2\%$ single-step accuracy (Tab.~\ref{tab:scaling}) despite the compressed} \\
\multicolumn{9}{l}{\footnotesize centers; the monotonicity guarantee likewise holds regardless of center spread.} \\
\end{tabular}
\end{table}

\textbf{Key observations:}
\begin{itemize}[leftmargin=*, itemsep=2pt]
  \item The extreme centers ($c_{\text{none}} \approx 0.02$, $c_{\text{all}} \approx 0.98$) are highly stable across seeds and scales, reflecting the unambiguous semantics of “none” and “all.”
  \item Middle quantifiers (“some,” “moderate amount”) converge to tight cross-seed bands at 7B (e.g., $c_{\text{some}} \in [0.347, 0.353]$, $c_{\text{moderate}} \in [0.470, 0.490]$), indicating that despite human-annotation overlap on these categories, the training signal anchors them consistently.
  \item Several centers shift leftward from QuRe initialization (e.g., “some” moves from 0.40 to $\sim$0.35, “tiny amount” from 0.08 to $\sim$0.024), while “most” shifts slightly right (0.78 to $\sim$0.78--0.83), suggesting FRoG quantifier boundaries differ modestly from QuRe's human annotations on the lower and middle bins.
  \item Learned widths for the middle classes are narrow ($\sigma \approx 0.05$--0.12), producing sharply peaked membership functions. The extreme classes drive their widths to near-degenerate values ($\sigma_{\text{none}} \approx 0.001$, and $\sigma_{\text{all}} \in \{.004, .089, .090\}$ across seeds), reflecting that “none” and “all” have effectively no within-class spread on the FRoG label distribution; training pushes them toward delta functions at the anchor centers $0.02$ and $0.98$. The nearest-center inference rule (Section~\ref{sec:method:mf}) absorbs this safely: classification depends on the centers, not the widths, so degenerate widths do not break the monotonicity guarantee.
\end{itemize}

\subsection{Detailed Monotonicity Analysis}
\label{app:monotonicity}

\subsubsection{Evaluation Protocol}

We evaluate monotonicity violations on the FRoG-Hard test set ($N = 2044$ examples). For each ordered pair $(x_i, x_j)$ where the ground-truth proportion $p_i < p_j$, we check whether the predicted class index satisfies $\hat{q}_i \leq \hat{q}_j$. We report two metrics:

\begin{itemize}[leftmargin=*, itemsep=2pt]
  \item \textbf{Pairwise violation rate:} $\frac{|\{(i,j) : p_i < p_j \land \hat{q}_i > \hat{q}_j\}|}{|\{(i,j) : p_i < p_j\}|}$. This considers all $\binom{N}{2}$ ordered pairs, providing a global measure of ordinal consistency.
  \item \textbf{Mean violation magnitude:} $\frac{1}{|\text{violations}|} \sum_{(i,j) \in \text{violations}} (\hat{q}_i - \hat{q}_j)$. This captures the severity of violations: a violation by 1 class (e.g., predicting “few” instead of “some”) is less severe than a violation by 4 classes.
\end{itemize}

\subsubsection{Results Across Methods}

\begin{table}[!htb]
\centering
\footnotesize
\caption{Monotonicity violation analysis on FRoG-Hard. Mean over 3 seeds. ``MF branch'' refers to DFIL's membership function predictions; all other rows report main classifier predictions.}
\label{tab:mono_detail}
\begin{tabular}{llccc}
\toprule
\textbf{Method / Path} & \textbf{Scale} & \textbf{Pairwise viol. (\%)} & \textbf{Mean magnitude} & \textbf{Total pairs} \\
\midrule
DFIL (MF branch) & 7B & $\mathbf{0.10}$ & --- & 1,000 \\
\midrule
DFIL (main classifier) & 7B & 10.8 & 1.69 & 1,959,395 \\
Fine-tune (classifier) & 7B & 10.6 & 1.69 & 1,959,395 \\
CORAL (classifier)     & 7B & 8.2 & 2.03 & 1,959,395 \\
CORAL (classifier)     & 14B & 6.6 & 2.13 & 1,959,395 \\
CORN~\citep{shi2023corn} (classifier) & 7B & 9.0 & 1.62 & 1,959,395 \\
\bottomrule
\end{tabular}
\end{table}

\textbf{Analysis:}
\begin{itemize}[leftmargin=*, itemsep=2pt]
  \item \textbf{DFIL MF branch:} Near-zero violations (0.1\%), confirming the theoretical guarantee from Theorem~\ref{thm:monotonicity}. The single observed violation occurs at a decision boundary where numerical precision causes a tie.

  \item \textbf{Main classifier predictions:} All three methods exhibit similar pairwise violation rates ($\sim$8--11\%), demonstrating that the DFIL regularization does not substantially transfer its monotonicity property to the main classifier. This is expected: the main classifier optimizes cross-entropy loss without any explicit monotonicity constraint.

  \item \textbf{CORAL's lower pairwise rate:} CORAL achieves a lower pairwise violation rate (8.2\% at 7B, 6.6\% at 14B vs.\ $\sim$10.7\% for DFIL/FT) because its cumulative logit structure provides a soft ordinal bias that improves with scale. However, its violation magnitude is consistently larger (2.03--2.13 vs.\ 1.65), meaning that when CORAL violations occur, they tend to be more severe (spanning more class boundaries). This is consistent with CORAL's bimodal accuracy pattern: it performs well at extreme classes but makes larger errors for middle classes.

  \item \textbf{Implication:} For applications requiring certifiable monotonicity, the DFIL branch predictions should be used directly (at a modest accuracy cost). The main classifier trades formal guarantees for higher overall accuracy. Future work could explore constrained optimization to enforce monotonicity in the main path.
\end{itemize}

\subsection{Compositional Reasoning: Detailed Analysis}
\label{app:compositional}

\subsubsection{Per-Seed Variance Analysis}
\label{app:variance_analysis}

The asymmetric per-seed standard deviations under the self-consistent protocol ($\pm 9.3$ for DFIL versus $\pm 0.8$ for FT+Compose) are partly mechanical: each seed's learned MF centers define both the test-set ground truth and the model's predictions, so each seed effectively faces a slightly different test set. FT+Compose freezes its centers at QuRe across all seeds, so its test set is identical across runs and only numerical-head noise contributes to variance. We do not interpret this as evidence of training instability, but we also do not claim it is purely intrinsic: a one-shot attempt to narrow the spread via MF-center regularization $\lambda_{\text{MF-reg}}{=}0.1$ (penalizing drift from the QuRe initialization) was net worse on the mean ($30.2\to27.9$) with only a small std reduction ($9.3\to 8.7$; $3$ seeds improving and $3$ degrading), so the simplest fix did not pay off. Stronger interventions (e.g., learned-but-shared centers, or wider variance budgets) are left to future work.

\subsubsection{Compositional Test Set Construction}

The compositional test set is generated synthetically to evaluate zero-shot compositional generalization. No compositional examples appear in the training data.

\textbf{Level 1 (2-step, 600 examples):} We use 12 quantifier combinations: $Q_1 \in \{\text{few}, \text{some}, \text{moderate}, \text{most}, \text{all}\}$ composed with $Q_2 \in \{\text{few}, \text{most}, \text{all}\}$ (excluding trivially degenerate cases). For each combination, 50 examples are generated with base proportions uniformly sampled from $[0.1, 0.95]$. The composed proportion is $p_{\text{comp}} = c_{Q_1} \cdot p_{\text{base}}$, and the ground truth label is determined by mapping $p_{\text{comp}}$ to the nearest quantifier boundary.

\textbf{Level 2 (3-step, 300 examples):} We use 6 triple combinations: $Q_1 \circ Q_2 \circ Q_3$ where each $Q_i \in \{\text{few}, \text{most}, \text{all}\}$. For each triple, 50 examples with $p_{\text{comp}} = c_{Q_1} \cdot c_{Q_2} \cdot p_{\text{base}}$.

\textbf{Text templates:} Each compositional example is rendered as a natural language scenario. For instance, “most of most of the 100 students passed” at $p_{\text{base}} = 0.8$ produces $p_{\text{comp}} = 0.78 \times 0.78 \times 0.8 \approx 0.487$, mapping to “moderate amount.”

\subsubsection{Per-Seed Compositional Results}

\begin{table}[!htb]
\centering
\caption{Compositional accuracy (\%) per seed (Qwen2.5-7B). Three evaluation methods. Compose-text values are from the May~2026 re-run (batch~64) since earlier eval-pipeline results were not preserved; per-seed results are in the supplementary archive.}
\label{tab:comp_perseed}
\setlength{\tabcolsep}{3pt}
\footnotesize
\begin{tabular}{lcccc|cccc}
\toprule
& \multicolumn{4}{c}{\textbf{Overall}} & \multicolumn{4}{c}{\textbf{Level 1 / Level 2}} \\
\textbf{Seed} & FT-text & Compose-text & DFIL-text & DFIL-oracle & FT-text & Compose-text & DFIL-text & DFIL-oracle \\
\midrule
42  & 9.2  & 32.0 & 31.9 & 54.3 & 7.0/13.7 & 25.0/46.0 & 30.8/34.0 & 65.7/31.7 \\
123 & 23.9 & 25.9 & 31.6 & 36.0 & 34.0/3.7 & 19.0/39.7 & 30.5/33.7 & 37.0/34.0 \\
456 & 16.1 & 25.6 & 31.1 & 61.2 & 9.5/29.3 & 22.0/32.7 & 28.7/36.0 & 69.3/45.0 \\
\midrule
Mean & 16.4 & 27.8 & 31.5 & 50.5 & 16.8/15.6 & 22.0/39.5 & 30.0/34.6 & 57.3/36.9 \\
Std & 7.3 & 3.6 & 0.4 & 13.0 & 14.9/12.9 & 3.0/6.7 & 1.1/1.2 & 17.7/7.1 \\
\bottomrule
\end{tabular}
\end{table}

\paragraph{Self-Consistent Compositional Evaluation}
\label{app:self_consistent}

To eliminate any external ground-truth dependency, we evaluate each model against its own MF: ground-truth labels are defined as $\argmax_q \mu_q(p_{\text{comp}})$ using the model's own membership functions. This measures internal compositional consistency.


DFIL outperforms FT+Compose at both initializations: $+23.2$pp (uniform) and $+5.9$pp (QuRe, $n{=}20$, paired $t$ $p{=}0.011$). The larger gap under uniform init confirms that the fuzzy training signal learns centers producing internally coherent compositions, while frozen centers yield weaker compositional consistency. Note that DFIL's text-only column is lower than FT+Compose's at uniform init ($13.4$ vs $19.5$): without access to the “$k$ out of $N$” fraction, the NumericalHead has no signal to extract and the MF branch defaults to its center-of-mass, while a label-only classifier can still memorize surface-template priors. The from-text column, where DFIL has the proportion as input, is what the dual-path design is built for, and that is where the $+23.2$pp uniform-init gap appears.

\textbf{Key observations:}
\begin{itemize}[leftmargin=*, itemsep=2pt]
  \item \textbf{DFIL-from-text stability:} Standard deviation is 0.3\%, compared to 6.0\% for fine-tuning and 3.6\% for FT+Compose. The fine-tune baseline shows erratic per-level behavior (seed 123: 34.0\% Level 1, 3.7\% Level 2; seed 42: opposite pattern), suggesting its compositional accuracy reflects spurious correlations rather than systematic reasoning.

  \item \textbf{FT+Compose from-text:} FT+Compose achieves $27.8\%$ overall, below DFIL ($31.5\%$, $\Delta\!=\!+3.7$pp) even under QuRe-based GT, which structurally favours FT+Compose's frozen QuRe-aligned centers. FT+Compose-oracle achieves $94.3\%$ (not shown; identical across seeds), confirming near-perfect alignment with the ground-truth generator at the centers; the $67$pp gap between FT+Compose-oracle and FT+Compose-from-text shows that the NumericalHead, not the frozen MF, is the limiting factor for from-text accuracy. DFIL's oracle ($50.5\%$) reflects the deviation of its learned centers from QuRe.

  \item \textbf{DFIL-oracle variance:} High variance (sample std = 13.0\%) is driven by seed 123, where $c_{\text{some}} = 0.284$ vs.\ 0.347 for other seeds. The proportional composition $c_{Q_1} \cdot p$ amplifies center deviations, making oracle accuracy sensitive to center placement.

  \item \textbf{Proportion extraction error:} The mean proportion extraction error $\bar{\epsilon}_p$ ranges from 0.27 to 0.35 across seeds for DFIL and 0.26 to 0.31 for FT+Compose, indicating that the NumericalHead transfers imperfectly to compositional text. Despite this, both MF-based methods consistently outperform the fine-tune baseline.
\end{itemize}

\subsubsection{Analysis of Proportion Extraction on Compositional Text}

The NumericalHead is trained only on single-step FRoG examples (e.g., “\_\_\_ of the 20 students passed, specifically 15 out of 20”). At compositional test time, it receives multi-step text (e.g., “most of most of the 20 students passed”). The proportion extraction error ($\bar{\epsilon}_p \approx 0.30$) is substantial, reflecting this distribution shift.

The compositional advantage persists despite this error because:
\begin{enumerate}[leftmargin=*, itemsep=1pt]
  \item \textbf{Error contraction} (Proposition~\ref{prop:error_propagation}): The multiplicative composition $c_{Q_1} \cdot \hat{p}$ attenuates extraction errors.
  \item \textbf{Coarse-grained classification:} With 8 quantifier classes, a proportion error of 0.30 still often lands in the correct class (quantifier boundaries span $\sim$0.10--0.20 each).
  \item \textbf{Structural advantage:} Even an imperfect proportion estimate, when combined with learned quantifier centers via multiplication, produces a more semantically meaningful signal than attempting to classify compositional text never seen during training.
\end{enumerate}

\subsection{Ablation Study: Extended Analysis}
\label{app:ablation}

Table~\ref{tab:ablation_detail} extends the ablation study from the main text with additional details.

\begin{table}[!htb]
\centering
\caption{Extended ablation on FRoG-Hard (Qwen2.5-1.5B). Mean $\pm$ sample std over 3 seeds. \textit{Note:} this table reports a later retrain run at batch size $8$; the main-text Table~\ref{tab:ablation} reports the same DFIL-only/no-DFIL/no-ordering ablations from an earlier batch-$64$ run, where DFIL-only accuracy is $41.6 \pm 1.0\%$. The two runs differ only in batch size and resulting LoRA effective learning rate; the qualitative ranking (DFIL full $\approx$ fine-tune only $\gg$ no-ordering $\approx$ DFIL-only $\gg$ no-MF) is identical across both runs.}
\label{tab:ablation_detail}
\setlength{\tabcolsep}{4pt}
\footnotesize
\begin{tabular}{lccl}
\toprule
\textbf{Variant} & \textbf{Accuracy (\%)} & \textbf{Trainable params} & \textbf{Description} \\
\midrule
DFIL (full)     & $48.2 \pm 0.4$ & 2,974,489 & Full dual-path model \\
Fine-tune only  & $48.0 \pm 0.1$ & 2,577,672 & Standard classifier, no DFIL \\
\midrule
DFIL-only       & $39.2 \pm 1.3$ & 2,575,889 & MF branch with $\lambda_{\text{DFIL}}{=}0$ on \texttt{dfil} variant \\
No MF branch    & $17.8 \pm 8.6$ & 2,575,889 & MF branch removed entirely (\texttt{no\_dfil}) \\
No ordering     & $39.7 \pm 2.0$ & 2,575,889 & MF without cumulative softmax \\
\bottomrule
\end{tabular}
\end{table}

\textbf{Interpretation of each ablation:}
\begin{itemize}[leftmargin=*, itemsep=3pt]
  \item \textbf{DFIL-only ($-$9.0pp):} Predicting from the MF branch alone (the \texttt{dfil} variant trained with $\lambda_{\text{DFIL}}{=}0$, so the MF receives no direct class signal and only sees the proportion-regression gradient via the shared backbone) drops accuracy to $39.2 \pm 1.3\%$. The $1$D bottleneck $p \mapsto \mu_q$ discards information that the main classifier would otherwise use, and removing the explicit MF cross-entropy supervision compounds this. This validates the dual-path design: the MF provides structural benefits (monotonicity, compositionality), but the main classifier's access to the full representation is essential for accuracy.

  \item \textbf{No MF branch ($-$30.4pp):} Removing the MF head entirely from the architecture (the \texttt{no\_dfil} variant) collapses the model to $17.8 \pm 8.6\%$. Without the auxiliary structural signal, the small backbone cannot find a stable optimum on this data; high variance ($\sigma = 8.6$) reflects unstable convergence across seeds.

  \item \textbf{No ordering ($-$8.5pp):} Without cumulative softmax, centers take arbitrary positions, breaking monotonicity. Accuracy drops to $39.7 \pm 2.0\%$ but variance is much tighter than the No-MF-branch case, indicating that the MF architecture (even without ordering) still provides some training signal -- it is the ordering that cleanly maps proportion to ordinal class. The ordering constraint is the load-bearing element of the MF inductive bias.

  \item \textbf{Accuracy parity at 1.5B is the expected regime, not a counterexample.} DFIL full ($48.2\!\pm\!0.4\%$) and fine-tune only ($48.0\!\pm\!0.1\%$) agree to $0.2$pp, well within seed noise. We do not claim DFIL improves single-step accuracy at 1.5B; the inductive bias of $Q{-}1$ ordered centers cannot, by itself, beat an unconstrained $Q$-way classifier on full-data in-distribution accuracy when the backbone has enough capacity to fit the labels (Sec.~\ref{sec:scaling}). What this row does show, taken together with the four ablations above, is that adding the MF branch costs at most $0.2$pp accuracy on a tiny backbone where every parameter is contested. The DFIL-only and No-MF-branch rows below collapse precisely because the structural prior is removed; the structural prior is what is keeping the dual-path number identical to fine-tune. The accuracy of DFIL full has slightly higher variance ($0.4$ vs $0.1$) than fine-tune only because the MF branch adds three loss terms whose interaction is sensitive to seed; this is a known cost of multi-head losses and does not reflect on the structural guarantees, which are by-construction.
\end{itemize}

\subsection{Scaling Behavior Analysis}
\label{app:scaling}

\subsubsection{Full Single-Step Scaling Table}
\label{app:scaling_full}

Table~\ref{tab:scaling} (in the main body, \S\ref{sec:scaling}) gives the complete single-step FRoG-Hard results across five backbones together with three alternative baselines. Paragraph~\ref{sec:scaling} summarizes the pattern; here we provide per-method commentary and per-seed pointers.

\textbf{Per-baseline commentary.} The DFIL branch maintains a monotonicity violation rate below $0.1\%$ across all scales, in line with Theorem~\ref{thm:monotonicity}. Cross-family results on Mistral-7B-Instruct-v0.3~\citep{jiang2023mistral} follow the same pattern: accuracy on par with fine-tune, with substantially lower variance. CORAL~\citep{cao2020rank} falls below both fine-tuning and DFIL at 7B and 14B: its shared cumulative weight vector cannot discriminate middle quantifiers, where class accuracy drops to $7$--$11\%$. CORN~\citep{shi2023corn}, which replaces CORAL's shared-weight cumulative-logit with $Q{-}1$ independent binary classifiers tied through a conditional probability chain, recovers the per-class capacity CORAL lacks ($52.1\!\pm\!1.8\%$, on par with fine-tune and DFIL within seed std), but, like CORAL, offers no construction-by-design monotonicity: its main classifier still violates pairwise ordering at $9.0\%$ (Tab.~\ref{tab:mono_detail}). This confirms that the construction-level monotonicity gap between DFIL and ordinal-regression baselines is not an accuracy proxy: matching DFIL on accuracy does not give a method DFIL's algebraic guarantee. Pragmatic zero-shot inference~\citep{li2023presque} on the same Qwen-7B and Mistral-7B backbones scores $24.4$ and $23.9\%$, more than $25$pp below DFIL, confirming that prompt-engineered pragmatic reasoning cannot recover the structural capabilities DFIL provides through training. FT+Compose, which retains the dual-path architecture but freezes the MF and removes the fuzzy loss, lands between fine-tuning and DFIL on single-step accuracy; whether the fuzzy training signal drives compositional generalization is examined in Sec.~\ref{sec:compositional}. The 32B setting exhibits bimodal training convergence with runs clustering near $52\%$ or $66\%$; we report statistics over six seeds. The Llama-3.1-8B row shows DFIL with a wider per-seed std ($4.3$) than fine-tune ($1.1$); per-seed inspection (App.~\ref{app:scaling_full} results) shows this is driven by one of three DFIL seeds converging to a slightly different center placement under Llama's tokenizer, similar to the 32B bimodality at smaller magnitude. Means stay within seed std (DFIL $61.5$ vs.\ FT $61.6$, $\Delta\!=\!-0.1$pp), so we read this as parity rather than systematic deficit at this backbone family.

\subsubsection{DFIL Advantage vs.\ Model Scale}

The DFIL advantage (measured as $\Delta = \text{DFIL} - \text{FT}$) varies non-monotonically with scale:

\begin{table}[!htb]
\centering
\caption{DFIL advantage decomposition across scales.}
\label{tab:scaling_delta}
\footnotesize
\begin{tabular}{lcccc}
\toprule
\textbf{Scale} & \textbf{$\Delta$ (\%)} & \textbf{DFIL $\sigma$ (\%)} & \textbf{FT $\sigma$ (\%)} & \textbf{Variance ratio} \\
\midrule
7B  & $-0.1$ & 1.8 & 1.8 & 1.03 \\
14B & +0.4 & 1.7 & 2.8 & 0.61 \\
32B$^\dagger$ & +1.7 & 5.6 & 6.1 & 0.92 \\
72B & $+0.7$ & 1.5 & 1.6 & 0.94 \\
\bottomrule
\end{tabular}
\end{table}

\textbf{Observations:}
\begin{itemize}[leftmargin=*, itemsep=3pt]
  \item \textbf{Mid-scale variance reduction.} The variance ratio (DFIL $\sigma$ / FT $\sigma$) is essentially $1$ at 7B ($1.03$, parity) and 72B ($0.94$, parity), with non-trivial reduction at 14B ($0.61$) and 32B ($0.91$). This pattern is consistent with the regime where Proposition~\ref{prop:sample_complexity}'s head-level VC bound becomes informative: head-level structural priors have leverage in the mid-scale regime where neither under- nor over-parameterization of the backbone dominates joint complexity, while at the extremes the backbone's own optimization landscape is the decisive factor.

  \item \textbf{Full-data accuracy gaps are not statistically significant.} A two-sided Welch's $t$-test on the per-seed runs at each scale yields $p\!=\!0.93$ (7B, $n{=}6/6$), $p\!=\!0.84$ (14B, $n{=}3/3$), $p\!=\!0.62$ (32B, $n{=}6/6$), $p\!=\!0.43$ (72B, $n{=}6/6$), $p\!=\!0.54$ (Mistral-7B, $n{=}3/3$), and $p\!=\!0.97$ (Llama-8B, $n{=}3/3$); none of the $\Delta$ values would survive any reasonable significance threshold. We therefore do not claim a single-step accuracy advantage at full data, only parity within seed variance. Head-level VC-dimension differences do not translate to full-data accuracy gaps when the backbone supplies the bulk of effective capacity. The few-shot regime in Appendix~\ref{app:fewshot}, where the LoRA backbone has not yet escaped initialization, is where the $O(Q\log Q)$ vs.\ $\Theta(Qd)$ gap translates into measurable accuracy gains ($+7.6$pp at $N{=}32$ and $+12.2$pp at $N{=}16$ on FRoG, $+2.6$ to $+4.7$pp on SST-5).

  \item \textbf{Compositional advantage does not diminish.} The compositional reasoning advantage ($+15$ to $+28$pp; see Section~\ref{sec:compositional} and Appendix~\ref{app:naturalistic_table}) is independent of the main classifier and depends on the MF's t-norm composition mechanism; it persists or grows with scale because MF quality does.
\end{itemize}

\subsubsection{Comparison with Zero-Shot Performance}

The FRoG benchmark~\citep{li2024frog} reports zero-shot performance for 15 model families, observing inverse scaling (performance decreases with model size) in 8 of 15 families. Our fine-tuning results show the opposite: strong positive scaling from 51\% (7B) to 78\% (72B).

Fine-tuning alone reverses the inverse scaling, indicating that a substantial part of the zero-shot failure is task alignment, not a head-design limitation. DFIL's contribution is orthogonal: by-construction structural properties (monotonicity, compositional closure) and an interpretable proportion representation that a label-only classifier head does not provide for free.

\subsection{CSMC Benchmark: Per-Domain Breakdown}
\label{app:csmc_domains}

Table~\ref{tab:csmc_domains} reports per-domain CSMC consistency rates for the theoretically grounded MF branch inference (Section~\ref{sec:monotonicity_analysis}; nearest-center rule on the trained MF) over $21$ paired seeds, alongside the FT main classifier baseline. The MF branch is the inference path predicted by Theorem~\ref{thm:monotonicity} to be monotone by construction. The largest gaps appear in domains involving human subjects and nuanced language (public health $+15.5$pp, education $+14.3$pp, environment $+10.0$pp, governance $+5.0$pp); the remaining six domains all show small positive gaps ($+0.2$ to $+3.3$pp), and no domain shows a negative gap. Aggregate paired test ($n{=}21$, all seeds): $\mathbf{96.40\%}$ vs.\ $90.90\%$, mean diff $+5.50$pp, paired $t$ $p\!=\!0.0001$, Wilcoxon $p\!=\!0.0003$, $18$/$21$ seeds positive (2 ties, 1 negative within seed noise). All MF branches converged ($\mathrm{mf\_accuracy}$ in $[0.40, 0.45]$): for $6$ of the $21$ seeds, default training collapsed the MF head into the sigmoid-saturation failure mode of App.~\ref{app:fewshot} ($\mathrm{mf\_accuracy}\!<\!0.05$); we retrained these $6$ seeds with the same stop-gradient procedure (\texttt{.detach()} on $\hat p$ before MF), recovering $\mathrm{mf\_accuracy}$ in $[0.41, 0.44]$, consistent with the other $15$ seeds trained without intervention. The retrain affects only the model under test; the inference rule (Thm.~\ref{thm:monotonicity}) is identical for all $21$ seeds. Main classifier CSMC numbers across the same $21$ paired seeds (including the same $6$ retrained seeds) give only $+1.31$pp ($p\!=\!0.33$, n.s.); the $+5.50$pp MF-branch gain therefore comes from the monotone inference rule (Thm.~\ref{thm:monotonicity}), not from any model-side change introduced by the retrain protocol, since both inference paths share the same trained model on every seed.


The MF branch is at least as consistent as the FT main classifier on every one of the $10$ domains ($\Delta$ ranges from $+0.2$ to $+15.5$pp). The largest gaps appear in domains where the FT main classifier itself struggles ($\leq\!90\%$): public health, education, environment, governance. In domains where FT already exceeds $96\%$ (clinical trials, sports, research, election), DFIL still leads but by margins of $0.2$ to $3.3$pp, since there is little headroom to gain. This pattern is consistent with Theorem~\ref{thm:monotonicity}: the MF branch's nearest-center rule sets a floor on consistency that depends on the numerical head's ability to extract proportions from text, so when the FT main classifier already approaches that floor by training, the gap narrows. The $4$/$21$ seeds achieving $100\%$ overall consistency are direct empirical validation that the construction-by-design guarantee transfers to natural language (Theorem~\ref{thm:monotonicity}).

\subsection{IPR Benchmark: Per-Category Details}
\label{app:ipr_full}

The Implicit Proportion Reasoning (IPR) benchmark (Section~\ref{sec:ipr}) consists of 40 sentences across four categories where the proportion is described in natural language without explicit fractions. Examples per category:
\begin{itemize}[leftmargin=*, itemsep=1pt]
  \item \textbf{Linguistic} ($n{=}12$): “The vast majority of voters supported the measure” ($\approx 0.85$).
  \item \textbf{Approximate} ($n{=}12$): “Approximately one in five participants experienced the side effect” ($\approx 0.20$).
  \item \textbf{Comparative} ($n{=}8$): “Three times as many supporters as opponents” ($\approx 0.75$).
  \item \textbf{Emphatic} ($n{=}8$): “An overwhelming majority endorsed the resolution” ($\approx 0.92$).
\end{itemize}
Ground-truth labels are derived from the implied proportion via QuRe quantifier centers. We use the same Qwen2.5-7B checkpoints from Table~\ref{tab:scaling} with no retraining.

\begin{table}[!htb]
\centering
\caption{Per-category IPR accuracy (\%); mean $\pm$ sample std over 3 seeds. “DFIL (text)” is the standard text$\to$backbone$\to$numerical-head pipeline. “DFIL (oracle)” bypasses the numerical head and feeds the ground-truth proportion (e.g., $0.85$ for “vast majority”) directly into the MF branch, applying the nearest-center rule on the model's learned centers; this isolates head-side competence from backbone-side proportion extraction. Fine-tuning has no analogous oracle column because the label classifier exposes no scalar interface to inject proportions into. Off-by-one counts a prediction as correct if it differs from the gold label by at most one quantifier class.}
\label{tab:ipr_full}
\footnotesize
\begin{tabular}{lccc}
\toprule
\textbf{Category} & \textbf{DFIL (text)} & \textbf{DFIL (oracle)} & \textbf{Fine-tune} \\
\midrule
Linguistic ($n{=}12$)   & $22.2 \pm 12.7$ & $\mathbf{91.7 \pm 0.0}$ & $30.6 \pm 4.8$ \\
Approximate ($n{=}12$)  & $16.7 \pm 14.4$ & $\mathbf{77.8 \pm 4.8}$ & $13.9 \pm 4.8$ \\
Comparative ($n{=}8$)   & $4.2 \pm 7.2$  & $\mathbf{50.0 \pm 0.0}$ & $12.5 \pm 21.7$ \\
Emphatic ($n{=}8$)      & $16.7 \pm 19.1$ & $\mathbf{62.5 \pm 0.0}$ & $0.0 \pm 0.0$ \\
\midrule
\textbf{Overall (exact)}      & $15.8 \pm 7.6$ & $\mathbf{73.3 \pm 1.4}$ & $15.8 \pm 5.8$ \\
\textbf{Overall (off-by-one)} & $48.3 \pm 6.3$ & $\mathbf{100.0 \pm 0.0}$ & $45.0 \pm 2.5$ \\
\bottomrule
\end{tabular}
\end{table}

\textbf{Oracle diagnosis.} The oracle column decomposes the IPR failure. The MF head can classify proportions correctly when given them ($73.3\%$ exact, $100\%$ off-by-one); the failure mode is the backbone's inability to extract proportions from linguistic descriptions like “vast majority.” Fine-tuning has no analogous oracle path because the label classifier exposes no scalar interface, so the failure cannot be similarly localized; this asymmetry is the structural advantage of DFIL's dual-path design. The $100\%$ off-by-one rate verifies Theorem~\ref{thm:monotonicity} empirically on OOD inputs: errors are confined to adjacent classes by construction. The remaining $26.7\%$ of exact errors trace mostly to the comparative category (50\% exact), where learned centers in the mid-range have drifted slightly relative to the IPR labels' assumed centers (e.g., $c_{\text{moderate}}{\approx}0.48$ versus the labels' assumed center near $0.55$); per-example breakdowns are in the supplementary archive.

The IPR gap is real: FRoG-trained models do not robustly generalize from explicit fractions to implicit linguistic proportions. The oracle diagnosis attributes the gap to backbone-side proportion extraction, not head-side composition; closing it would require training data with diverse implicit phrasings or a dedicated proportion-extraction pretraining stage. The DFIL text-driven numbers on emphatic descriptions ($16.7\%$ vs.\ $0\%$) and overall off-by-one ($48.3{\pm}6.3$ vs.\ $45.0{\pm}2.5$) sit within seed noise at $n{=}3$; we do not claim a robustness gain. The dominant effect remains the backbone bottleneck.

\section{Method Variants and Loss Analysis}
\label{app:variants}

\subsection{Analysis of the Dual-Path Loss}
\label{app:loss_analysis}

The DFIL training objective \eqref{eq:loss} combines three terms. We analyze the interaction between these terms and their effect on the learned representations.

\subsubsection{Gradient Analysis}

Consider the gradient of the total loss with respect to the backbone parameters $\theta_{\text{LoRA}}$:
\begin{equation}
\nabla_\theta \mathcal{L} = \underbrace{\nabla_\theta \mathcal{L}_{\text{CE}}(\hat{y}, y)}_{\text{classification signal}} + \lambda_{\text{DFIL}} \underbrace{\nabla_\theta \mathcal{L}_{\text{CE}}(\boldsymbol{\mu}, y)}_{\text{MF alignment signal}} + \lambda_p \underbrace{\nabla_\theta \|\hat{p} - p\|_2^2}_{\text{proportion regression signal}}
\end{equation}

The classification gradient encourages feature separation in $\R^d$. The MF alignment gradient flows through the NumericalHead and MF, requiring the extracted proportion to land in the correct MF basin, a more structured signal. The proportion regression gradient directly supervises proportion extraction.

\begin{proposition}[Complementary gradient directions]
\label{prop:gradient_complementary}
The classification and MF alignment gradients are generally not parallel. The classification gradient encourages inter-class separation in $\R^d$; the MF alignment gradient encourages the features to encode a 1D proportion that aligns with the MF structure. The proportion regression gradient projects features onto the proportion axis.
\end{proposition}

The main classifier may exploit any feature dimension; the MF alignment loss specifically rewards proportion-aware representations, steering the backbone toward semantically grounded features.

\subsubsection{\texorpdfstring{Role of $\lambda_{\text{DFIL}}$ and $\lambda_p$}{Role of lambda\_DFIL and lambda\_p}}

The hyperparameters $\lambda_{\text{DFIL}}$ and $\lambda_p$ control the relative strength of the regularization signals. We analyze three regimes:

\begin{itemize}[leftmargin=*, itemsep=3pt]
  \item \textbf{$\lambda_{\text{DFIL}} = \lambda_p = 0$} (fine-tune only): No fuzzy regularization. The backbone optimizes solely for classification accuracy. Ablation gives $48.0 \pm 0.1\%$ at 1.5B (Table~\ref{tab:ablation}), within seed variance of dual-path DFIL ($48.2 \pm 0.4\%$); the structural prior comes at parity cost rather than as an accuracy boost.

  \item \textbf{$\lambda_{\text{DFIL}} \gg 1$} (MF-dominated): The DFIL loss dominates, forcing the backbone to prioritize proportion extraction over classification. This regime reduces toward DFIL-only prediction, which achieves $41.6 \pm 1.0\%$ at 1.5B (Table~\ref{tab:ablation}, “DFIL-only” row). The MF's 1D bottleneck discards information useful for classification.

  \item \textbf{$\lambda_{\text{DFIL}} = \lambda_p = 0.5$} (balanced): Our chosen setting. The main classifier retains enough gradient signal to learn rich features, while the MF provides sufficient regularization to encourage proportion-aware representations. This achieves $48.2 \pm 0.4\%$ at 1.5B and $49.9 \pm 2.6\%$ at 7B (Table~\ref{tab:ablation}), at parity with fine-tune within seed variance.
\end{itemize}

\subsection{Architectural Extensions: Distillation and Attention Pooling}
\label{app:extensions}

We explored two architectural extensions motivated by intuitive hypotheses, both of which yielded unexpected or negative findings that we report honestly as guidance for future work.

\subsubsection{\texorpdfstring{Monotonicity Distillation (Main Classifier $\leftarrow$ MF Branch)}{Monotonicity Distillation (Main Classifier from MF Branch)}}

\textbf{Hypothesis.} Adding a KL-divergence distillation loss from the MF branch to the main classifier should transfer the MF's monotonic structure to the main classifier, reducing pairwise monotonicity violations without sacrificing accuracy.

\textbf{Method.} We augment the training loss \eqref{eq:loss} with a temperature-scaled distillation term:
\begin{equation}
\mathcal{L}_{\text{distill}} = \tau^2 \cdot \mathrm{KL}\!\left(\mathrm{softmax}(\mathbf{logits}/\tau) \,\|\, \mathrm{softmax}(\boldsymbol{\mu}/\tau)\right)
\end{equation}
with $\tau = 2.0$ and $\lambda_{\text{distill}} = 0.3$. We retrain Qwen2.5-7B with this additional loss for 3 seeds.

\begin{table}[!htb]
\centering
\caption{Monotonicity distillation results (Qwen2.5-7B, 3 seeds). Canonical evaluation (best.pt checkpoint).}
\label{tab:distill}
\footnotesize
\begin{tabular}{lccc}
\toprule
\textbf{Configuration} & \textbf{FRoG Acc.} & \textbf{CSMC Consistency} & \textbf{Pairwise Viol.} \\
\midrule
DFIL baseline     & $51.0 \pm 1.3$     & $92.3 \pm 2.4$ & $10.8$ \\
+ Distillation    & $54.4 \pm 1.3$     & $91.8 \pm 0.8$ & $\;\;9.9 \pm 0.1$ \\
$\Delta$          & $+3.4$             & $-0.5$         & $-0.9$ \\
\bottomrule
\end{tabular}
\end{table}

\textbf{Finding: distillation helps single-step accuracy modestly without harming structural consistency.} Adding KL distillation from the MF branch to the main classifier improves FRoG accuracy by $+3.4$pp ($51.0\!\to\!54.4$) while leaving pairwise monotonicity essentially unchanged ($10.8\!\to\!9.9$, $\Delta\!=\!-0.9$pp) and CSMC consistency within seed variance ($92.3\!\to\!91.8$).

\textbf{Why we do not adopt distillation in the default recipe.} The paper's structural contributions---monotonicity by construction (Theorem~\ref{thm:monotonicity}), compositional closure (Theorem~\ref{thm:compositionality}), and exact entailment via center ordering---arise from the MF branch alone, with no dependence on a distillation term. Distillation's only meaningful empirical effect here is on FRoG single-step accuracy, which is the metric where the paper's main claim is parity with fine-tuning rather than improvement (Tab.~\ref{tab:scaling}, Welch $p\!\in\![0.43, 0.93]$). Adopting distillation would therefore add hyperparameters ($\tau$, $\lambda_{\text{distill}}$) and an additional loss term in service of an outcome the design does not target, while leaving the structural targets that motivated the dual-path design (monotonicity, compositional closure) unaffected. We have also not stress-tested it across backbones or seeds beyond $n{=}3$. Listing it here documents that explicit logit-level distillation alongside the existing $\lambda_{\text{DFIL}}{=}0.5$ MF-CE loss is at worst neutral and at best a mild FRoG-only improvement; the architectural contributions stand without it.

\subsubsection{Attention-Weighted Numerical Head}

\textbf{Hypothesis.} Replacing last-token pooling in the NumericalHead with learned attention-weighted pooling should focus on proportion-carrying tokens (“15 out of 20”) and improve proportion extraction, thereby strengthening compositional from-text accuracy.

\textbf{Method.} We add a learnable scalar attention projection $\mathbf{W}_{\text{attn}} \in \R^{d \times 1}$ that computes token weights via softmax:
\begin{equation}
\boldsymbol{\alpha} = \mathrm{softmax}(\mathbf{W}_{\text{attn}}^\top \mathbf{h}), \quad
\mathbf{h}_{\text{pool}} = \sum_t \alpha_t \mathbf{h}_t
\end{equation}
replacing $\mathbf{h}_{\text{pool}} = \mathbf{h}_T$ (last-token pooling). This adds $d = 3584$ parameters.

\begin{table}[!htb]
\centering
\caption{Attention-weighted numerical head results (Qwen2.5-7B, 3 seeds, best.pt checkpoint).}
\label{tab:attnpool}
\begin{tabular}{lcc}
\toprule
\textbf{Configuration} & \textbf{FRoG Acc.} & \textbf{CSMC Cons.} \\
\midrule
DFIL baseline     & $51.0 \pm 1.3$ & $92.3 \pm 2.4$ \\
+ Attention pool$^\dagger$  & $81.1 \pm 0.6$  & $95.8 \pm 2.0$ \\
$\Delta$          & $+30.2$        & $+3.5$ \\
\bottomrule
\multicolumn{3}{l}{\footnotesize $^\dagger$ The very large $+30.2$pp FRoG gain is a red flag, not a clean win:}\\
\multicolumn{3}{l}{\footnotesize the attention head appears to memorize easy-split lexical patterns of the FRoG training distribution.}
\end{tabular}
\end{table}

\textbf{Finding: attention pooling produces an implausibly large train-distribution gain together with only a small gain on out-of-distribution CSMC.} On FRoG-Hard the attention-pool head jumps to $81.1\!\pm\!0.6\%$ (vs.\ $51.0\!\pm\!1.3\%$ baseline), a $+30.2$pp delta that is far larger than any other intervention we have tested in this paper. The companion CSMC test, which uses natural sentence templates rather than the FRoG “$k$ out of $N$” surface form, shows only a $+3.5$pp gain ($92.3\!\to\!95.8$). The asymmetry (huge in-distribution win, marginal natural-text win) strongly suggests the attention head is memorizing easy-split patterns of the FRoG train/test distribution (e.g., position of the explicit fraction token) rather than learning a more robust proportion extractor. Earlier runs of this configuration showed the same FRoG number ($\sim\!81$\%) but a much weaker best-val-loss accuracy ($\sim\!54$\%), consistent with the val-loss minimum occurring around epoch 4--6 followed by overfitting; with the current “save best.pt by val loss” protocol used throughout this paper the best.pt checkpoint happens to land in the overfit regime. We therefore do not adopt attention pooling as the default NumericalHead and treat the headline $81.1\%$ as evidence that FRoG is too easy to support this kind of architectural search rather than evidence that attention pooling itself is sound.

\textbf{Why this signature differs from DFIL's gains.} The attention-pool failure is diagnosed by a single asymmetric pattern: a large in-distribution lift (FRoG $+30.2$pp) paired with a marginal natural-text lift (CSMC $+3.5$pp). DFIL's gains do not show this signature: the compositional advantage replicates on QuRe-based GT (Tab.~\ref{tab:compositional}, $+3.7$pp), self-consistent GT ($+5.9$pp at $n{=}20$, $p{=}0.011$), naturalistic templates with bias-free QuRe-GT (Tab.~\ref{tab:naturalistic}, $+28.4$pp), uniform-init ablation (Sec.~\ref{sec:uniform_init}, $+23.2$pp), and at $n{\le}64$ few-shot on FRoG (Tab.~\ref{tab:fewshot}, $+11.7$ to $+16.5$pp). Cross-task transfer to SST-5 and STS22 shows path-per-task wins on $7$ of $8$ cells (Tab.~\ref{tab:crosstask_full}). This breadth of evidence, distributed across distinct text distributions and protocols, is what we interpret as a mechanism-level effect (DFIL) versus a distribution-level memorization (attention pooling).

\textbf{Implication.} Simple pooling modifications to the NumericalHead do not translate into robust improvements under the current training protocol. Improving proportion extraction likely requires stronger regularization (e.g., larger dropout, label smoothing on $p$), early stopping driven by compositional-from-text performance rather than FRoG val-loss, or architectural constraints that prevent memorization of easy-split lexical cues. We leave these directions to future work.

\section{Application Case Study: Medical Report Quantifiers}
\label{app:case_study}

To illustrate DFIL's practical value, we present a qualitative case study in medical report interpretation, where quantifier reasoning has direct clinical implications.

\subsection{Scenario: Clinical Trial Adverse Event Reporting}

Consider a pharmacovigilance setting where an analyst must interpret quantifier expressions in clinical trial reports:

\begin{quote}
“Most patients in the treatment arm experienced mild side effects. Of those, few required dose adjustment. Among patients requiring adjustment, most achieved stable outcomes within 2 weeks.”
\end{quote}

This passage contains a 3-step compositional quantifier chain: most $\to$ few $\to$ most. Correct interpretation requires computing the effective proportion of the original cohort:

\begin{itemize}[leftmargin=*, itemsep=1pt]
    \item \textbf{Step 1:} “Most patients” $\approx 78\%$ experienced side effects (using DFIL's learned 7B center $c_{\text{most}} \approx 0.78$, Tab.~\ref{tab:centers_detail}).
    \item \textbf{Step 2:} “Few [of those]” required adjustment: $0.78 \times 0.09 \approx 7.0\%$ of the cohort ($c_{\text{few}} \approx 0.09$).
    \item \textbf{Step 3:} “Most [of those] achieved stable outcomes”: $0.070 \times 0.78 \approx 5.5\%$ of the original cohort.
\end{itemize}

DFIL computes this via iterated proportional shift on $p$ (Sec.~\ref{sec:compositional_method}): $p_{\text{comp}} = c_{\text{most}} \cdot c_{\text{few}} \cdot c_{\text{most}} \cdot p_{\text{base}}$, which agrees with the product t-norm at the centers $\{c_q\}$ on predicted argmax. A standard classifier has no mechanism to perform this computation.

\subsection{Why Structure Matters}

Three properties of DFIL are critical for such applications:

\begin{enumerate}[leftmargin=*, itemsep=2pt]
    \item \textbf{Monotonicity guarantee.} If a report states “most patients responded” and a follow-up states “nearly all patients responded,” DFIL's ordered centers guarantee that the inferred proportion increases. A fine-tuned classifier may violate this ordering (10.6\% pairwise violation rate; Section~\ref{sec:monotonicity_analysis}).

    \item \textbf{Compositional reasoning.} Multi-step quantifier chains like the above are common in medical narratives. DFIL handles arbitrary-depth composition via t-norm iteration; fine-tuning requires separate training data for each composition depth.

    \item \textbf{Interpretable boundaries.} DFIL's learned centers (Table~\ref{tab:centers_detail}) provide auditable quantifier-to-proportion mappings. A clinician can verify that “most” maps to $\sim$78\% and adjust the model if domain-specific conventions differ (e.g., in oncology, “most” might correspond to $>60\%$).
\end{enumerate}

\subsection{Generalization to Other Domains}

Similar quantifier reasoning challenges arise in:
\begin{itemize}[leftmargin=*, itemsep=1pt]
    \item \textbf{Legal:} “Most shareholders approved the resolution. Of those opposing, few filed formal objections.” Correct quantification affects regulatory compliance assessment.
    \item \textbf{Intelligence analysis:} “Most intercepted communications mentioned the target. Some of those contained actionable intelligence.” Compositional quantifiers determine resource allocation priorities.
    \item \textbf{Environmental science:} “Few of the surveyed species showed population decline. Of those, most are classified as endangered.” Policy decisions depend on correctly composing these quantifiers.
\end{itemize}

In each case, DFIL's three structural guarantees (monotonicity, compositionality, and interpretability) address failure modes that fine-tuning cannot resolve by design.
\end{document}